\documentclass[a4paper,11pt]{article}
\pdfoutput=1
\usepackage[margin=0.9in]{geometry}

\usepackage{subcaption}
\usepackage{graphicx} 
\usepackage{booktabs} 
\usepackage{amssymb}
\usepackage{times}
\usepackage{epsfig}
\usepackage{graphicx}
\usepackage{color}
\usepackage{url}
\usepackage{bbm}
\usepackage{multicol}

\providecommand{\dif}{\mathop{}\!\mathrm d}

\providecommand{\hide}[1]{}

\usepackage{amsmath,amsfonts}
\usepackage{amsopn,amssymb}

\usepackage{amsthm}

\theoremstyle{plain}
\newtheorem{thm}{Theorem}

\newcounter{claimcounter}
\numberwithin{claimcounter}{thm}
\newenvironment{claim}{\stepcounter{claimcounter}{\par\noindent\underline{Claim \theclaimcounter:}}}{}

\newcounter{claimproofcounter}
\numberwithin{claimproofcounter}{thm}
\newenvironment{claimproof}{\stepcounter{claimproofcounter}{\par\noindent\underline{Proof of Claim \theclaimproofcounter:}}}{}

\theoremstyle{definition}

\theoremstyle{remark}

\usepackage{bm} 
\usepackage{algorithmicx}
\usepackage{algorithm}
\PassOptionsToPackage{noend}{algpseudocode}
\usepackage{algpseudocode}
\algnewcommand{\LineComment}[1]{\Statex \(\triangleright\) #1}

\newcommand{\rpm}{\raisebox{.2ex}{$\scriptstyle\pm$}}
\newcommand{\vct}[1]{\boldsymbol{#1}} 
\newcommand{\mat}[1]{\boldsymbol{#1}} 

\newcommand{\field}[1]{\mathbb{#1}}
\newcommand{\R}{\field{R}} 
\newcommand{\T}{^{\textrm T}} 

\newcommand{\ProbOpr}[1]{\mathbb{#1}}

\newcommand{\expect}[2]{%
\ifthenelse{\equal{#2}{}}{\ProbOpr{E}_{#1}}
{\ifthenelse{\equal{#1}{}}{\ProbOpr{E}\left[#2\right]}{\ProbOpr{E}_{#1}\left[#2\right]}}} 
\newcommand{\var}[2]{%
\ifthenelse{\equal{#2}{}}{\ProbOpr{VAR}_{#1}}
{\ifthenelse{\equal{#1}{}}{\ProbOpr{VAR}\left[#2\right]}{\ProbOpr{VAR}_{#1}\left[#2\right]}}} 

\DeclareMathOperator{\argmax}{arg\,max}
\DeclareMathOperator{\argmin}{arg\,min}

\newcommand{\vy}{\vct{y}}

\newcommand{\vkappa}{\vct{\kappa}}

\newcommand{\zeros}{\vct{0}}

\newcommand{\mI}{\mat{I}}
\newcommand{\mK}{\mat{K}}

\newcommand{\Vor}{\mathrm{Vor}}
\algnewcommand{\algorithmicgoto}{\textbf{go to}}%
\algnewcommand{\Goto}[1]{\algorithmicgoto~\ref{#1}}%

\usepackage{xspace}
\newcommand{\ALGNAME}{{\sc{BOIDP}}\xspace}

\newcount\Comments  
\usepackage{array,multirow,graphicx}
\usepackage{color}
\usepackage{wrapfig}
\definecolor{darkgreen}{rgb}{0,0.5,0}
\definecolor{purple}{rgb}{1,0,1}
\newcommand{\kibitz}[2]{\ifnum\Comments=1\textcolor{#1}{#2}\fi}
\newcommand{\zw}[1]{\kibitz{purple}      {[ZW: #1]}}

\title{\LARGE \bf
Focused Model-Learning and Planning for \\Non-Gaussian Continuous State-Action Systems
}

\author{Zi Wang \quad Stefanie Jegelka \quad Leslie Pack Kaelbling \quad Tom\'as Lozano-P\'erez
\thanks{Computer Science and Artificial Intelligence Laboratory, Massachusetts Institute of Technology, 77 Massachusetts Ave., Cambridge, MA 02139.
        {\tt\small \{ziw,stefje,lpk,tlp\}@csail.mit.edu}}
}
\date{\vspace{-5ex}}

\makeatletter
\renewcommand{\ALG@beginalgorithmic}{\small}
\makeatother
\usepackage{caption}
\begin{document}
 \sloppy

\maketitle
\begin{abstract}
\normalsize
We introduce a framework for model learning and planning in stochastic domains with continuous state and action spaces and non-Gaussian transition models. It is efficient because (1) local models are estimated only when the planner requires them; (2) the planner focuses on the most relevant states to the current planning problem; and (3) the planner focuses on the most informative and/or high-value actions. Our theoretical analysis shows the validity and asymptotic optimality of the proposed approach. Empirically, we demonstrate the effectiveness of our algorithm on a simulated multi-modal pushing problem.


\end{abstract}

\section{Introduction}
\label{sec:intro}

Most real-world domains are sufficiently complex that it is difficult
to build an accurate deterministic model of the effects of actions.
Even with highly accurate actuators and sensors, stochasticity still widely appears in basic manipulations, especially non-prehensile ones~\cite{yu2016more}.  The
stochasticity may come from inaccurate execution of actions as well as
from lack of detailed information about the underlying world state. For
example, 
rolling a die is a deterministic process that depends
on the forces applied, air resistance, etc.; however, we are
not able to model the situation sufficiently accurately to plan
reliable actions, nor to execute them repeatably if we could plan them.
We can plan using a stochastic model of the system, but
in many situations, such as rolling dice or pushing a can shown in Fig.~\ref{fig:intro}, the
stochasticity 
is not modeled well by 
additive single-mode Gaussian noise, and a more sophisticated model class is necessary. 
\hide{
Our framework is modular, separating model-learning from planning. 
 Instead of trying to construct a
highly accurate physical simulator that will inevitably 
make inaccurate assumptions about the environment and the interaction
between robots and objects, e.g. Coulomb friction law and uniform
pressure distribution, w We propose a new planning algorithm, in which the non-Gaussian transition models are estimated on the fly from data. }\hide{ We estimate a stochastic model of the robot's
interaction with the environment from data. Then, to make plans using
the learned models, we propose a novel planning algorithm for non-Gaussian
continuous state-action systems.} 
%

In this paper, we address the problem of learning and planning for
non-Gaussian stochastic systems in the practical setting
of continuous state and action spaces. Our framework learns transition models that can be used for planning to achieve
different objectives in the same domain, as well as to be potentially
transferred to related domains or even different types of robots.
This strategy is in contrast to most reinforcement-learning
approaches, which build the objective into the structure being
learned.  In addition, rather than constructing a single monolithic
model of the entire domain which could be difficult to represent, our method uses a memory-based learning
scheme, and computes localized models on the fly, only when the
planner requires them. To avoid constructing models that do not contribute to improving the policy, the planner should focus only
on states relevant to the current planning problem, and actions that
can lead to high reward.

We propose a closed-loop planning algorithm that applies to stochastic
continuous state-action systems with arbitrary transition models. It
is assumed that the transition models are represented by a function
that may be expensive to evaluate. Via two important steps, we focus
the computation on the current problem instance, defined by the
starting state and goal region. To focus on relevant states, we use
real time dynamic programming (RTDP)~\cite{barto1995learning} on a set
of states strategically sampled by a rapidly-exploring random tree
(RRT)~\cite{lavalle1998rapidly,huynh2016incremental}. To focus
selection of actions from a continuous space, we develop a new batch
Bayesian optimization (BO) technique that selects and tests, in
parallel, action candidates that will lead most quickly to a
near-optimal answer.


\hide{

Given a stochastic model of state transitions in the continuous
state-action space, we need a strategy for planning.   
Recently, \cite{huynh2016incremental} introduced the incremental Markov
Decision process algorithm (iMDP), which incrementally generates a
sequence of Markov Decision Processes to approximate the original
continuous-time/space problem. iMDP assumes the process dynamics can
be described by a Wiener process, in which transition noise is
Gaussian with execution-time-dependent variance. It was shown theoretically and empirically for stochastic LQR problems that the value function of iMDP converges to the true value function. However, it is very challenging to apply iMDP to just a bit more complicated systems with non-Gaussian and non-closed-form dynamics.

Building on the partial success of iMDP, }
\begin{figure}[t]
        \centering
        \includegraphics[width=0.48\textwidth]{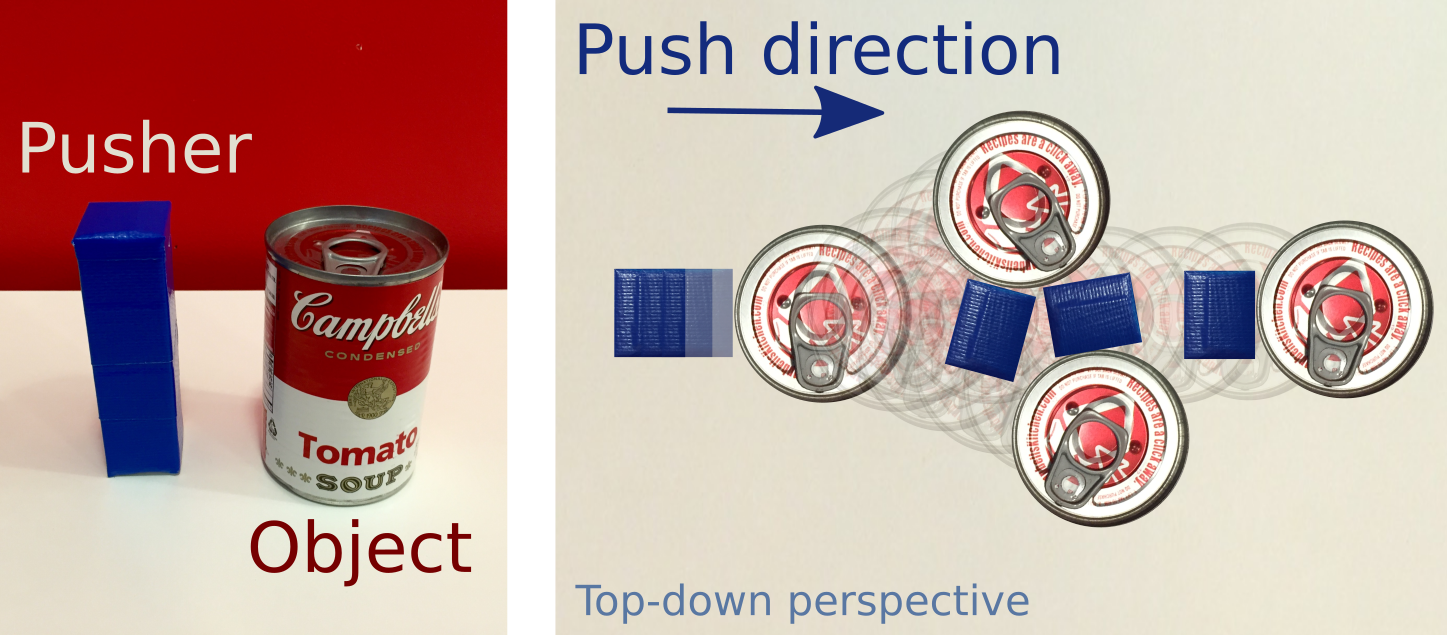}
        \caption{A quasi-static pushing
          problem: the pusher has a velocity controller with low
          gain, resulting in non-Gaussian transitions. We show
          trajectories for object and pusher resulting from the same
          push velocity.}
        \label{fig:intro} 
\end{figure}

We show theoretically that the expected accumulated difference between the optimal
value function of the original problem and the value of the policy we
compute vanishes sub-linearly in the number of actions we test, under
mild assumptions. Finally we evaluate our approach empirically on a
simulated multi-modal pushing problem, and demonstrate the effectiveness and efficiency of the proposed algorithm.




\section{Related Work}
\label{sec:related}

\paragraph{Learning}
The class of problems that we address may be viewed as
reinforcement-learning (RL) problems in observable continuous
state-action spaces.  It is possible to address the problem through
model-free RL, which estimates a value function or policy for
a specific goal directly through experience.  Though the majority of
work in RL addresses domains with discrete action spaces, there has
been a thread of relevant work on value-function-based RL in
continuous action spaces~\cite{Gullapalli94, Baird93, Prokhorov97,
  vanHasselt07,Nichols15}.  An alternative approach is to do direct
search in the space of
policies~\cite{deisenroth2011pilco,jakab2012reinforcement}.

In continuous state-action spaces, model-based RL, where
a model is estimated 
to optimize a policy, can often be more
effective.  Gaussian processes (GP) can help to learn the dynamics~\cite{deisenroth2008model,rottmann2009adaptive, nguyen2011model},
which can then be used by GP-based dynamic programming~\cite{deisenroth2009gaussian,rottmann2009adaptive} to determine a
continuous-valued closed-loop policy for the whole state space.  More details can be found in the excellent survey~\cite{ghavamzadeh2015bayesian}.

Unfortunately, the common assumption of i.i.d Gaussian noise on the
dynamics is restrictive and may not hold in
practice~\cite{yu2016more}, and the transition model can be
multi-modal. It may additionally be difficult to obtain a good GP
prior.  The basic GP model is can capture neither the
multi-modality nor the heteroscedasticity of the noise. While more
advanced GP algorithms may address these problems, they often suffer
from high computational
cost~\cite{titsias2011variational,yuan2009variational}.

Moldovan et al.~\cite{moldovan2015optimism}~addressed the problem of
multi-modality by using Dirichlet process mixture models (DPMMs) to
learn the density of the transition models. Their strategies for
planning were limited by deterministic assumptions, appropriate for
their domains of application, but potentially resulting in collisions
in ours.  Kopicki
et al.~\cite{kopicki2009prediction,kopicki2010prediction,kopicki2011learning}~addressed
the problem of learning to predict the behavior of rigid objects under
manipulations such as pushing, using kernel density estimation.
In this paper, we propose an efficient planner that can work with
arbitrary, especially multi-modal stochastic models in continuous state-action spaces. Our learning method in the experiment resembles DPMMs but we estimate the density on the fly
when the planner queries a state-action pair. We were not able to
compare our approach with DPMMs because we found DPMMs not
computationally feasible for large datasets.




\paragraph{Planning}
We are interested in domains for which 
queries are made by specifying a starting state and a goal set, and in which the solution
to the given query can be described by a policy that covers only a
small fraction of the state space that the robot is likely to encounter. 

Planning only in the fraction of the state-action space that the robot
is likely to encounter is, in general, very challenging.  Other
related work uses tree-based search
methods~\cite{weinstein2012bandit,mansley2011sample,yee2016monte}, where the
actions are selected by optimizing an optimistic heuristic. These
algorithms are impractical for our problem because of the exponential
growth of the tree and the lack of immediate rewards that can guide
the pruning of the
tree. 

In contrast to the tree-search algorithms, iMDP~\cite{huynh2016incremental}, which is most related to our work, uses sampling techniques from
RRTs to create successively more accurate discrete MDP approximations
of the original continuous MDP, ultimately converging to the optimal
solution to the original problem.  Their method assumes the ability to solve the Bellman equation optimally (e.g. for a simple stochastic LQR problem), the availability of the backward transition models, and that the dynamics is modeled by a Wiener process, in which the transition noise is
Gaussian with execution-time-dependent variance. However, the assumptions are too restrictive to model our domains of
interest where the dynamics is non-closed-form, costly to evaluate,  non-reversible, and non-Gaussian. Furthermore, iMDP is designed for stochastic control problems with multiple starting states and a single goal, while we are interested in multiple start-goal pairs.

Our work builds on the idea of constructing a sequence of MDPs from iMDP~\cite{huynh2016incremental}, and aims at practically resolving the challenges of state/action selection faced by both iMDP and tree-search-based planners~\cite{weinstein2012bandit}.

\paragraph{Bayesian optimization} There have been a number
of applications of BO in optimal control,
although to our knowledge, it has not been previously applied to action-selection in continuous-action MDPs.
BO has been used to find weights in a neural network controller~\cite{frean2008using}, to solve for the parameters of a hierarchical MDP~\cite{brochu2009}, and to address safe exploration in finite MDPs~\cite{turchetta2016safe}. To our knowledge, BO has not been previously applied to action-selection in continuous-action MDPs.



\section{Problem formulation}
\label{sec:background}

Let the state space $S\subset \R^{d_s}$ with metric $d$ and the control space
$U\subset \R^{d_u}$ both be compact and measurable sets. The interior of the state space $S$ is $S^o$ and the boundary is $\partial S$. For the control space $U$, there exists an open set $U^o$ in $R^{d_u}$ such that $U$ is the closure of $U^o$. We assume the state
is fully observed (any remaining latent state will manifest as
stochasticity in the transition models).  Actions $a = (u, \Delta t)$
are composed of both a control on the robot and the duration for which
it will be exerted, so the action space is $A=U\times [T_{min},T_{max}]$, where $T_{min}, T_{max}\in\R_{+}\setminus \{\infty\}$ are the minimum and the maximum amount of duration allowed. The action space $A$ is also a compact set. 
 The starting state is $s_0$, and the goal
region is denoted as $\mathcal G \subset S$, in which all states are terminal states. We assume $\mathcal G$ has non-zero measure, and $S$ has finite measure. The {\em transition model}
has the form of a continuous probability density function $p_{s' \mid s,a}$ on the
resulting state $s'$, given previous state $s$ and action $a$,
such that $\forall s'\in S, p_{s' \mid s,a}(s' \mid s,a)\geq 0, \int_S
p(s' \mid s,a)\dif s' = 1$. 

Given a transition model and a cost function $C: S\times S\times A\rightarrow \R$ associated with a goal region
, we can formulate the problem 
as a continuous state-action MDP $(S, A, p_{s' \mid s,a}, R, \gamma)$, where
$R(s' \mid s, a) = - C(s' \mid s,a)$ is the immediate reward function and $\gamma$ is
the discount factor. A high reward is assigned to the states in the goal region $\mathcal G$, and a cost is assigned to colliding with obstacles or taking any action. We would like to solve for the optimal policy $\pi:S\rightarrow A$, for which the value of each state $s$ is
$$ V^\pi(s) =\underset{a\in A}{\max}\int_{s'\in S} p_{s' \mid s,a}(s' \mid s,a)
  \left(R(s'\mid s,a) + \gamma^{\Delta t} V^\pi(s')\right) \dif s'.$$

\hide{
It is common practice to approach the solution of continuous MDPs by
constructing and solving approximations in the form of discrete MDPs~\cite{Alterovitz07,huynh2016incremental}. iMDP~\cite{huynh2016incremental} pursued this approach
for stochastic motion-planning domains, using RRTs with backward
extension to sample states for successive discrete approximate MDPs, and
obtaining probabilistic completeness by covering the possible paths
with probability 1 while keeping at least one goal state always in
the state set. However, sampling states and planning with iMDP is still difficult for many real-world problems because of the 
excessive assumptions  such as Gaussian transition noise and the availability of backward transition models . We describe a similar but more practical state sampling approach in Sec.~\ref{sssec:sampling}. %
\hide{
iMDP~\cite{huynh2016incremental} pursued this approach
for stochastic motion-planning domains, using RRTs with backward
extension to sample states for successive discrete approximate MDPs, and
obtaining probabilistic completeness by covering the possible paths
with probability 1 while keeping at least one goal state always in
the state set. In this paper, we adopt a similar strategy, but using RRT with forward extension to handle domains 
However, the approach proposed by iMDP only works in restrictive conditions where the backward transition model is available. }
Once we have a discrete-state continuous-action MDP
$(\tilde S, A, \hat{P}_{s' \mid s,a}, R, \gamma)$, 
we can solve for a
deterministic policy via value iteration. Here $\tilde S$ is a finite
set of states that satisfy $\tilde S \subset S$, $\hat{P}_{s' \mid s,a}(s' \mid s,a)$
is the probability mass function for the transition from state
$s\in \tilde S$ and action $a=(u,\Delta t)\in A$ to a new state $s'\in \tilde S$,
$R(s' \mid s,a)$ is the immediate reward of the transition for
$s,s'\in \tilde S, a\in A$, and $\gamma$ is the discount factor of the
rewards, the same as in the continuous case. 
Given a discrete MDP and policy $\pi$, the value of each state $s$ at
iteration $n$ is

\begin{small}
$${\small {\tilde V^\pi(s) = \sum_{s'\in\tilde S} \hat{P}_{s' \mid s,a}(s' \mid s,\pi(s))
  \left(R(s' \mid s,\pi(s)) + \gamma^{\Delta t^\pi} \tilde V^\pi(s')\right)}.}$$
\end{small}
Note that $\Delta t$ is part of $\pi(s)$. For the optimal policy $\pi^*$,
\begin{small}
$$\tilde V^{\pi^*}(s) = \underset{a\in A}{\max}
  \sum_{s'\in\tilde S} \hat{P}_{s' \mid s,a}(s' \mid s,a) \left(R(s' \mid s,a) + \gamma^{\Delta
  t} \tilde V^{\pi^*}(s')\right),$$ 
  \end{small}
 which is not a trivial optimization problem since 
it may be computationally expensive to obtain
$\hat P_{s' \mid s,a}(\cdot \mid \cdot,\cdot)$. We develop a new method for batch Bayesian optimization
with caching to efficiently approximate this problem, which is detailed in Sec.~\ref{sec:bayesopt}.
}


\section{Our method: \ALGNAME}
\label{sec:imdp}
We describe our algorithm Bayesian Optimization Incremental-realtime Dynamic Programming (\ALGNAME) in this section. 
At the highest level, \ALGNAME in Alg.~\ref{alg:boidp} 
operates in a loop, in which it samples a
discrete set of states $\tilde S\subset S$ and attempts to solve the discrete-state,
continuous-action MDP $\tilde {\mathcal M} = (\tilde S, A, \hat{P}_{s'\mid s,a}, R, \gamma)$. 
Here $\hat{P}_{s' \mid s,a}(s' \mid s,a)$ is the probability mass function for the transition from state
$s\in \tilde S$ and action $a\in A$ to a new state $s'\in \tilde S$. 
The value function for the optimal policy of the approximated MDP $\tilde {\mathcal M}$ is
$ V(s) =\underset{a\in A}{\max} \; Q_s(a),$ where
\begin{align}Q_s(a)=\sum_{s'\in \tilde S} \hat P_{s'\mid s,a}(s' \mid s,a)
  \left(R(s' \mid s,a) + \gamma^{\Delta t} V(s')\right).
  \label{qfunc}
  \end{align}
 If the value of the resulting policy is satisfactory according to the task-related stopping criterion\footnote{For example, one stopping criterion could be the convergence of the starting state's value $V(s_0)$. },
we can proceed; otherwise, additional state samples are added and 
the process is repeated. 
Once we have a policy $\pi$ on $\tilde S$ from RTDP, the robot can iteratively 
obtain and execute the policy for the nearest state to the current state in the sampled set $\tilde S$ by the metric $d$.

There are a number of challenges 
underlying each step of \ALGNAME. First, we need to
find a way of accessing the transition probability density function
$p_{s' \mid s,a}$ ,
which 
is critical for the approximation of
$\hat{P}_{s' \mid s,a}(s' \mid s,a)$ and the value function. We
describe our ``lazy access'' strategy in
Sec.~\ref{ssec:model}. Second, we must find a way to 
compute the values of as few states as possible
to fully
exploit the ``lazy access'' to the transition model. Our solution is
to first use an RRT-like
process~\cite{lavalle1998rapidly,huynh2016incremental} to generate the
set of states that asymptotically cover the state space with low
dispersion (Sec.~\ref{ssec:sampling}), and then ``prune'' the
irrelevant states via RTDP~\cite{barto1995learning}
(Sec.~\ref{ssec:rtdp}). Last, each dynamic-programming
update in RTDP requires a maximization over the action space; we
cannot achieve this analytically and so must sample a finite set of
possible actions. We develop a new batch BO algorithm to focus action
sampling on regions of the action space that are informative and/or
likely to be high-value, as described in Sec.~\ref{sec:bayesopt}.

Both the state sampling and transition estimation
processes assume a collision checker
$\textsc{ExistsCollision}(s,a,s')$ that checks the path from $s$ to
$s'$ induced by action $a$ for collisions with permanent objects in
the map.

\begin{algorithm}[t]
\scriptsize
\caption{\ALGNAME}
\label{alg:boidp}
\begin{algorithmic}[1]
\Function{\ALGNAME}{$s_0, \mathcal G, S, p_{s' \mid s,a}, N_{\min}$}
\State $\tilde S\gets \{s_0\}$
\Loop
\State $\tilde S \gets \textsc{SampleStates}(N_{\min},\tilde S, \mathcal G, S, p_{s' \mid s,a})$
\State $\pi, V = \textsc{RTDP}(s_{0},\mathcal G, \tilde S, p_{s' \mid s,a})$
\EndLoop
\State \textbf{until} stopping criteria reached
\State \textsc{ExecutePolicy}($\pi, \tilde S, \mathcal G$)
\EndFunction
\Statex
\Function{ExecutePolicy}{$\pi, \tilde S, \mathcal G$}

\Loop
 \State $s_c\gets$ current state
\State $\tilde s \gets \argmin_{s\in \tilde S} d(s,s_c)$
\State Execute $\pi(\tilde s)$
\EndLoop
\State \textbf{until} current state is in $\mathcal G$
\EndFunction
\end{algorithmic}
\end{algorithm}

\subsection{Estimating transition models in \ALGNAME}
 
In a typical model-based learning approach, first a monolithic model
is estimated from the data and then that model is used to construct a
policy.  Here, however, we aim to scale to large spaces with
non-Gaussian dynamics, a setting where it is very difficult to represent and
estimate a single monolithic model.  Hence, we take a different approach via 
``lazy access'' to the model: we
estimate local models on demand, as the planning process
requires information about relevant states and actions. 
\label{ssec:model}
  \begin{algorithm}[t]
  \caption{Transition model for discrete states}
  \label{alg:transition}
  \begin{algorithmic}[1]
  \Function{TransitionModel}{$s,a,\tilde S,p_{s'  \mid  s,a}$}
  \State $\hat S\gets \textsc{HighProbNextStates}(p_{s' \mid s,a}(\tilde S  \mid  s,a)) \cup \{s_{obs}\}$
  \Comment$s_{obs}$ is a terminal state
  \State $\hat P_{s'  \mid  s,a}(\hat S  \mid  s,a) \gets p_{s'  \mid 
    s,a}(\hat S  \mid  s,a)$
  \For{$s'$ \textbf{in} $\hat S$}
  \If{$s'\in S^o$ \textbf{and} \textsc{ExistsCollision}($s,a, s'$)}
  \State $\hat P_{s'  \mid  s,a}(s_{obs}  \mid  s,a) \gets \hat P_{ s'
     \mid  s,a}(s_{obs}  \mid  s,a) + \hat P_{s'  \mid  s,a}(s'  \mid  s,a) $
  \State $\hat P_{s'  \mid  s,a}(s'  \mid  s,a) \gets 0$
  \EndIf
  \EndFor
  \State $\hat P_{s' \mid  s,a}(\hat S  \mid  s,a) \gets
  \textsc{Normalize}(\hat P_{s' \mid  s,a}(\hat S  \mid  s,a))$
  
  \State \Return $\hat S, \hat P_{s'  \mid  s,a}(\hat S  \mid  s,a)$
  \EndFunction
  \end{algorithmic}
  \end{algorithm}

We assume a dataset $D = \{s_i, a_i, s'_i\}_{i=0}^N$ for the system
dynamics and the dataset is large enough to provide a good
approximation to the probability density of the next state given any
state-action pair. If a stochastic simulator exists for the transition
model, one may collect the dataset dynamically in response to queries
from \ALGNAME.
The ``lazy access'' provides a flexible
interface, which can accommodate a variety of different
density-estimation algorithms with asymptotic theoretical guarantees, 
such as kernel density
estimators~\cite{wied2012consistency} and Gaussian mixture
models~\cite{moitra2010settling}. 
In our experiments, we
focus on learning Gaussian mixture models with the assumption that
$p_{s' \mid s,a}(s' \mid s,a)$ is distributed according to a mixture of Gaussians $\forall (s,a)\in S\times A$. 

Given a discrete set of states $\tilde S$, starting state $s$
and action $a$, we compute the approximate discrete transition model
$\hat P_{s' \mid s,a}$ as shown in Algorithm~\ref{alg:transition}. We use the function  $\textsc{HighProbNextStates}$ to select the largest set of next states
$\hat S\subseteq \tilde S$ such that $\forall s'\in \hat S, p_{s' \mid s,a}(s' \mid s,a)>\epsilon$. $\epsilon$ is a small threshold parameter, e.g. we can set $\epsilon=10^{-5}$. 
If $p_{s'  \mid  s,a}$ does not take obstacles into account, we have to
check the path from state $s$ to next state $s'\in \tilde S$ induced
by action $a$ for collisions, and model their effect in the
approximate discrete transition model $\hat P_{s'  \mid  s,a}$. To achieve
this, we add a dummy terminal state $s_{obs}$, which represents a
collision, to the selected next-state set $\hat S$.  Then, for any
$s, a, s'$ transition that generates a collision, we move the
probability mass $\hat P_{s'  \mid  s,a}(s'  \mid  s,a)$ to the transition to
the collision state $\hat P_{s'  \mid  s,a}(s_{obs}  \mid  s,a)$. Finally,
$\hat P_{s' \mid  s,a}(\hat S \mid  s,a)$ is normalized and returned together with
the selected set $\hat S$.

These approximated discrete transition models can be indexed by state $s$
and action $a$ and cached for future use in tasks that use the
same set of states $\tilde S$ and the same obstacle map.  
The memory-based essence of our modeling strategy is similar to the
strategy of non-parametric models such as Gaussian processes, which
make predictions for new inputs via smoothness assumptions and
similarity between the query point and training points in the data set.

For the case where the dynamics model $p_{s' \mid s,a}$ is given, computing the approximated transition  $\hat P_{s' \mid s,a}$ could still be computationally expensive because of the collision checking. Our planner is designed to alleviate the high computation in $\hat P_{s' \mid s,a}$ by focusing on the relevant states and actions, as detailed in the next sections.

\subsection{Sampling states}
\label{ssec:sampling}
Algorithm~\ref{alg:statesample} describes the state
sampling procedures. The input to
$\textsc{SampleStates}$ in Alg.~\ref{alg:statesample} includes the minimum number of states,
$N_{\min}$, to sample at each iteration of \ALGNAME. 
It may be that more than
$N_{\min}$ states are sampled, because sampling must continue until at
least one terminal goal state is included in the resulting set $\tilde S$. 
To generate a discrete state set, we sample 
states both in the interior of $S^o$ and on its boundary 
$\partial S$.  Notice that we can always add more states by
calling \textsc{SampleStates}.
\begin{algorithm}[t]
\scriptsize
\caption{RRT states sampling for \ALGNAME}
\label{alg:statesample}
\begin{algorithmic}[1]
\Function{SampleStates}{$N_{min}, \tilde S, \mathcal G, S, p_{s' \mid s,a}$}
\State $\tilde S^o \gets \textsc{SampleInteriorStates}(\lceil N_{min}/2\rceil, \tilde S, \mathcal G, S, p_{s' \mid s,a})$
\State $\partial\tilde{ S}\gets \textsc{SampleBoundaryStates}(\lceil N_{min}/2 \rceil, \tilde S, \mathcal G, S, p_{s' \mid s,a})$
\State \Return $\tilde S^o\cup \partial\tilde{ S}$
\EndFunction
\Statex
\Function{SampleInteriorStates}{$N_{min}, \tilde S, \mathcal G, S, p_{s' \mid s,a}$}
  \While{$ \mid \tilde S \mid  < N_{min}$ or $\mathcal G \cap \tilde S = \emptyset$} 
  \label{alg1for}
    \State $s_{rand} \gets \textsc{UniformSample}(S)$ \label{alg1uniform}
    \State $s_{nearest} \gets \textsc{Nearest}(s_{rand}, \tilde S)$ \label{alg1nearest}
    \State $s_n, a_n \gets \textsc{RRTExtend}(s_{nearest}, s_{rand}, p_{s' \mid s,a})$
    \If{found $s_n,a_n$}
    \State $\tilde S\gets \tilde S \cup \{s_n\}$
    \EndIf
  \EndWhile
  \State \Return $\tilde S$
\EndFunction
\Statex
\hide{
\Function{RejectSampling}{ }
\Loop
\State $s_{rand}\gets QuasiMonteCarloSampler()$
\If{\textbf{not }$\textsc{ExistsCollision}(s_{rand})$}
\State \Return $s_{rand}$
\EndIf
\EndLoop 
\EndFunction
\Statex
}
\Function{RRTExtend}{$s_{nearest},s_{rand}, p_{s' \mid s,a}$}
\State $d_n = \infty$
\While{stopping criterion not reached}
\State $a \gets \textsc{UniformSample}(A)$ \label{alg1uniforma}
\State $s'\gets \textsc{Sample}\left(p_{s' \mid s,a}(\cdot \mid s_{nearest},a)\right)$ \label{alg:rrtuni}
\If{ $\left(\textbf{not } \textsc{ExistsCollision}(s,s',a)\right)$ \textbf{and} $d_n > d(s_{rand},s')$}
\State $d_n\gets d(s_{rand},s')$
\State $s_n, a_n \gets s', a$
\EndIf
\EndWhile
\State \Return $s_n,a_n$
\EndFunction
\end{algorithmic}
\end{algorithm}


To generate one interior state sample, we randomly generate a state
$s_{rand}$, and find $s_{nearest}$ that is the nearest state to
$s_{rand}$ in the current sampled state set $\tilde S$. Then we
sample a set of actions from $A$, for each of which we sample the next state $s_n$ from the dataset $D$ given the state-action pair $s_{neareast},a$ (or from $p_{s' \mid s,a}$ if given). We choose the action
$a$ that gives us the $s_n$ that is
the closest to $s_{rand}$. 
To sample states on the boundary $\partial S$, we assume a uniform random
generator for states on $\partial S$ is available. If not, we can use
something similar to $\textsc{SampleInteriorStates}$ but only sample inside
the obstacles uniformly in line~\ref{alg1uniform} of
Algorithm~\ref{alg:statesample}. Once we have a sample $s_{rand}$ in
the obstacle, we try to reach $s_{rand}$ by moving along the path
$s_{rand}\rightarrow s_n$ incrementally until a collision is reached.



\subsection{Focusing on the relevant states via RTDP}
\label{ssec:rtdp}
We apply our algorithm with a known starting state
 $s_0$ and  goal region $\mathcal G$.  Hence, it is not necessary to
 compute a complete policy, and so we can use RTDP~\cite{barto1995learning} to compute a value
function focusing on the relevant state space and a policy that, with high
probability, will reach the goal before it reaches a state for which
an action has not been determined. We assume an upper bound of the
values for each state $s$ to be $h_u(s)$. One can approximate $h_u(s)$ via the shortest distance from each state to the goal region on the fully connected graph with vertices $\tilde S$.  We show the pseudocode in Algorithm~\ref{alg:rtdp}.
\begin{algorithm}[t]
\scriptsize
\caption{RTDP for \ALGNAME}
\label{alg:rtdp}
\begin{algorithmic}[1]

\Function{RTDP}{$s_{0},\mathcal G, \tilde S, p_{s'| s,a}$}

\For{$s$ \textbf{in} $\tilde S$}
\State $V(s) = h_u(s)$ 
\Comment Compute the value upper bound 
\EndFor

\While{$V(\cdot)$ not converged}
\State $\pi,V\gets\textsc{TrialRecurse}(s_{0},\mathcal G, \tilde S, p_{s'| s,a})$
\EndWhile
\State\Return $\pi, V$
\EndFunction

\Statex
\Function{TrialRecurse}{$s,\mathcal G, \tilde S, p_{s'| s,a}$}
\If{reached cycle \textbf{or} $s\in\mathcal G$ } 
\State \Return 
\EndIf 

\State $\pi(s)\gets \argmax_a Q(s,a,\tilde S, p_{s'| s,a})$ \label{opt1}
\Comment Max via BO
\State $s' \gets \textsc{Sample}(\hat P_{s'| s,a}(\tilde S| s,\pi(s)))$
\State $\textsc{TrialRecurse}(s',\mathcal G, \tilde S, p_{s'| s,a})$
\State $\pi(s)\gets \argmax_a Q(s,a,\tilde S, p_{s'| s,a})$ \label{opt2}
\Comment Max via BO
\State $V(s)\gets Q(s,\pi(s),\tilde S, p_{s'| s,a})$

\State\Return $\pi, V$
\EndFunction
\Statex

\Function{Q}{$s,a,\tilde S, p_{s'| s,a}$}
\If{$\hat P_{s'| s,a}(\tilde S| s,a)$ has not been computed}
\State $\hat P_{s'| s,a}(\tilde S| s,a) = \zeros$ 
\Comment $\hat P_{s'| s,a}$ is a shared matrix 

\State 
\footnotesize{$\hat S, \hat P_{s'| s,a}(\hat S| s,a)$  $\gets \textsc{TransitionModel}(s,a,\tilde S,p_{s'| s,a})$}

\EndIf
\State \Return $R(s,a) + \gamma^{\Delta t} \sum_{ s'\in\hat S} \hat P_{s'| s,a}(s'| s,a)V(s')$
\EndFunction

\hide{
\Function{TransitionModel}{s,a}
\State $\hat S = \textsc{HighProbNextStates}(p_{s'| s,a}(\tilde S |  s, a))$ \label{md2}

\Loop
\State $\hat S = \textsc{HightProbNextStates}(p_{s'| s,a}(\tilde S |  s, a))$ \label{md2}
\If{$p_{s'| s,a}(\hat S |  s, a) >$ Threshold}
\State \textbf{break}
\Else
\State $\tilde S \gets \tilde S \cup \textsc{SampleInit}(p(s'| s,a))$ 
\EndIf
\EndLoop

\State $T(\hat S|  s, a) \gets \textsc{Normalize} (p_{s'| s,a}(\hat S |  s, a))$
\State \Return  $T(\hat S|  s, a)$
\EndFunction
}
\hide{
\Function{SampleInit}{$p$}
\Comment $p$ is a probability distribution over $S$
\State $s\gets \textsc{Sample}(p)$
\If{$s\notin \tilde S$}
\State \textsc{InitNode}(s)
\EndIf
\State \Return s
\EndFunction
}

\end{algorithmic}
\end{algorithm}
When doing the recursion (\textsc{TrialRecurse}), we can save additional
computation when maximizing $Q_s(a)$. Assume that the last time
$\argmax_a Q_s(a)$ was called, the result was $a^*$ and the
transition model tells us that $\bar S$ is the set of possible
next states. The next time we call $\argmax_a Q_s(a)$, if
the values for $\bar S$ have not changed, we can just return $a^*$ as
the result of the optimization. This can be done easily by caching the
current (optimistic) policy and transition model for each state.
\hide{

Because the space is continuous and there is always area that we have
not sampled before, which results in small probability even in the
high probability next states (line~\ref{md2} in
Algorithm~\ref{alg:rtdp}), and the states we sampled cannot really
represent the next state very well. In this case, because we would
like to make sure our sampled states are representative enough,
heuristically we can solve the following optimization instead of
$\max_a Q(s,a)$: 
\begin{equation}
\label{eq:maxq}
\begin{aligned}
& \underset{a\in A}{\textbf{maximize}}
& & Q(s,a)\\
& \textbf{subject to} 
&& p(\hat S|  s,a) > \delta
\end{aligned}
\end{equation}
We use the Lagrangian form of the optimization problem for our
Bayesian optimization. In this way we can ensure our anytime policy is
good enough for the current number of states.} 
\subsection{Focusing on good actions via BO}
\label{sec:bayesopt}
RTDP in Algorithm~\ref{alg:rtdp} relies on a challenging optimization over a continuous and possibly high-dimensional action space.  
Queries to
$Q_s(a)$ in Eq.~\eqref{qfunc} 
  can be very expensive because in many cases
a new model must be estimated. 
Hence, we need to limit the number of
points queried during the optimization.
There is no clear strategy for computing the gradient of
$Q_s(a)$, and random sampling is very sample-inefficient especially as
the dimensionality of the space grows.
We will view the optimization
of $Q_s(a)$ as a black-box function optimization problem, and use
batch BO to efficiently approximate the solution and make full use of the parallel computing resources.

\begin{algorithm}[b]
\scriptsize
  \caption{Optimization of $Q_s(a)$ via sequential GP optimization}\label{alg:bo}
  \begin{algorithmic}[1]
    \State $\mathfrak D_0 \gets \emptyset$
      \For{$t = 1\rightarrow T$}
      \State $\mu_{t-1}, \sigma_{t-1}$ $\gets$ GP-predict($\mathfrak D_{t-1}$)
      \State $a_t\gets \argmin_{a\in A}{\frac{h_u(s)-\mu_{t-1} (a)}{\sigma_{t-1} (a)}}$ 
      \State $y_t\gets Q_s(a_t)$ 
      \State $\mathfrak D_t \gets \mathfrak D_{t-1} \cup \{a_t,y_t\}$
      \EndFor
  \end{algorithmic}
\end{algorithm}

\begin{algorithm}[t]
\scriptsize
  \caption{Optimization of $Q_s(a)$ via batch GP optimization}\label{alg:bbo}
  \begin{algorithmic}[1]
    \State $\mathfrak D_0 \gets \emptyset$
      \For{$t = 1\rightarrow T$}
      \State $\mu_{t-1}, \sigma_{t-1}$ $\gets$ GP-predict($\mathfrak D_{t-1}$)
      \State $B\gets \emptyset$
      \For{$i = 1\rightarrow M$}
      \State $B\gets B\cup \{\argmax_{a\in A}{F_s(B\cup \{a\}) - F_s(B)} \}$ 
      \EndFor
      \State $\vy_B\gets Q_s(B)$ 
      \Comment Test $Q_s$ in parallel
      \State $\mathfrak D_t \gets \mathfrak D_{t-1}\cup\{B,\vy_B\}$
      \EndFor
  \end{algorithmic}
\end{algorithm}

We first briefly review a sequential Gaussian-process optimization
method, GP-EST~\cite{wang2016est}, shown in Algorithm~\ref{alg:bo}. For a
fixed state $s$, we assume $Q_s(a)$ is a sample from a Gaussian process with
zero mean and kernel $\kappa$. 
At iteration $t$, we select action $a_t$
and observe the function value $y_t=Q_s(a_t)+\epsilon_t$, where
$\epsilon_t\sim\mathcal N(0,\sigma^2)$. 
  Given the observations
$\mathfrak D_t=\{(a_\tau,y_\tau)\}_{\tau=1}^{t}$ up to time $t$, we
obtain the posterior mean and covariance of the $Q_s(a)$ function via
the kernel matrix
$\mK_t =\left[\kappa(a_i,a_j)\right]_{a_i,a_j\in \mathfrak D_t}$ and
$\vkappa_t(a) = [\kappa(a_i,a)]_{a_i\in \mathfrak
  D_t}$~\cite{rasmussen2006gaussian}:
\begin{eqnarray*}
\mu_{t}(a) & = & \vkappa_t(a)\T(\mK_t+\sigma^2\mI)^{-1}\vy_t\\
\kappa_{t}(a,a') & = & \kappa(a,a') - \vkappa_t(a)\T(\mK_t+\sigma^2\mI)^{-1}
\vkappa_t(a')\;\;.
\end{eqnarray*}
  The posterior variance is given by
$\sigma^2_{t}(a) = \kappa_t(a,a)$. We can then use the posterior mean
function $\mu_t(\cdot)$ and the posterior variance function
$\sigma^2_t(\cdot)$ to select which action to test in the next
iteration. We here make use of the assumption that we have an
upper bound $h_u(s)$ on the value $V(s)$. We select the action that is most likely to
have a value greater than or equal to $h_u(s)$ to be the next one to
evaluate. %
Algorithm~\ref{alg:bo} relies on sequential tests of
$Q_s(a)$, but 
it may be much more effective to test $Q_s(a)$ for multiple values of $a$
in parallel. This requires us to choose a diverse subset of actions
that are expected to be informative and/or have good values. 

We propose a new 
batch Bayesian optimization method that selects a query set that 
has large diversity and low values of the acquisition function
$G_{s,t} (a)= \left(\frac{h_u(s)-\mu_{t-1} (a)}{\sigma_{t-1} (a)}\right)$. The key idea is to
maximize a submodular objective function with a cardinality constraint
on $B\subset A, |B|=M$ that characterize both diversity and quality:
\begin{align}
F_s(B) = \log \det \mK_B  - \lambda \sum_{a\in B} \frac{h_u(s)-\mu_{t-1} (a)}{\sigma_{t-1} (a)}
\end{align}
where $\mK_B = [\kappa(a_i,a_j)]_{a_i, a_j\in B}$ and $\lambda$ is a
trade-off parameter for diversity and quality. If $\lambda$ is large,
$F_s$ will prefer actions with lower $G_{s,t(a)}$, which means a better chance of having high
values. If $\lambda$ is low, $\log \det \mK_B$ will dominate
$F_s$ and a more diverse subset $B$ is preferred. 
$\lambda$ can be chosen by cross-validation. We optimize the heuristic function
$F_s$ via greedy optimization which yield a $1-\frac1e$ approximation
to the optimal solution. We describe the batch GP optimization in
Algorithm~\ref{alg:bbo}. 

The greedy optimization can be efficiently implemented using the
following property of the determinant:
\begin{align}
& F_s(B\cup\{a\}) - F_s(B) \\
&= \log\det \mK_{B\cup\{a\}} - \log\det\mK_B - \frac{h_u(s)-\mu_{t-1} (a)}{\sigma_{t-1} (a)} \\
& = \log (\kappa_a - \vkappa_{Ba}\T\mK_B^{-1}\vkappa_{Ba}) - \frac{h_u(s)-\mu_{t-1} (a)}{\sigma_{t-1} (a)}
\end{align} 
where $\kappa_a= \kappa(a,a), \vkappa_{Ba} = [\kappa(a_i,a)]_{a_i\in B}$.

\section{Theoretical analysis}
In this section, we characterize the theoretical behavior of \ALGNAME.   Thm.~\ref{thm:ebd} establishes the error bound for the value function on the \emph{$\hat \pi^*$-relevant} set of states~\cite{barto1995learning}, where $\hat\pi^*$ is the optimal policy computed by \ALGNAME. A set $B\subseteq S$ is called \emph{$\pi$-relevant} if all the states in $B$ is reachable via finite actions from the starting state $s_0$ under the policy $\pi$. We denote $|\cdot|_{B}$ as the $L_{\infty}$ norm of a function $\cdot$ over the set $B$. 

We assume the existence of policies whose relevant sets intersect with $\mathcal G$. 
If there exists no solution to the continuous state-action MDP $\mathcal M$, our algorithm will not be able to generate an RRT whose vertices contain a state in the goal region $\mathcal G$, and hence no policy will be generated. We use the reward setup described in Sec.~\ref{sec:background}. For the simplicity of the analysis, we set the reward for getting to the goal large enough such that the optimal value function $V^*(s) = \max_{a\in A} Q_s(a)$ is positive for any state $s$ on the path to the goal region under the optimal policy $\pi^*$. 

We denote the measure for the state space $S$ to be $\rho$ and the measure for the action space $A$ to be $\psi$. Both $\rho$ and $\psi$ are absolutely continuous with respect to Lebesgue measure. The metric for $A$ is $g$, and for $S$ is $d$. We also assume the transition density function $p_{s'\mid s,a}$ is not a generalized function and satisfies the property that if $\int_{\mathcal F} p_{s'|s,a}(s'|s,a) \dif s' > 0$, then $\rho(\mathcal F)>0$. 
Without loss of generality, we assume $\min \Delta t = 1$ and $\max \Delta t = \mathcal T$.  

Under mild conditions on $Q_s(a)$ specified in Thm.~\ref{thm:ebd}, we show that with finitely many actions selected by
BO, the expected accumulated error expressed by the difference between the optimal value function 
$V^*$ and the value function $\hat V$ of the policy computed by
\ALGNAME in Alg.~\ref{alg:boidp} on the \emph{$\hat \pi^*$-relevant} set decreases sub-linearly in the number of actions selected for optimizing $Q_s(\cdot)$ in Eq.~\eqref{qfunc}.


\begin{thm}[Error bound for \ALGNAME]
\label{thm:ebd}
 Let $D = \{s_i, a_i, s'_i\}_{i=0}^N$ be the dataset that is collected from the true transition probability $p_{s' \mid s,a}$, $\forall (s,a)\in S\times A$. We assume that the transition model $p_{s' \mid s,a}$ estimated by the density estimator asymptoticly  converges to the true model. $\forall s\in S$, we assume $Q_s(a)=\int_{s'\in S} p_{s' \mid s,a}(s' \mid s,a)
  \left(R(s'\mid s,a) + \gamma^{\Delta t} V^*(s')\right) \dif s'$ is a function locally continuous at $\argmax_{a\in A} Q_s(a)$, where $V^*(\cdot) = \max_{a\in A}Q_{\cdot}(a)$ is the optimal value function for the continuous state-action MDP  $\mathcal M=(S,  A,
p_{s' \mid s,a}, R, \gamma)$. $V^*(\cdot)$ is associated with an optimal policy whose relevant set contains at least one state in the goal region $\mathcal G$. 
  At iteration $k$ of RTDP in Alg.~\ref{alg:rtdp}, we define $\hat V_k$ to be the value function for $\tilde{\mathcal M}=(\tilde S, A, \hat{P}_{s' \mid s,a}, R, \gamma)$ approximated by \ALGNAME, $\hat \pi_k$ to be the policy corresponding to $\hat V_k$, and $B_k$ to be the \emph{$\hat\pi_k$-relevant set}. We assume that
 \[Q_{s,k}(a) =  \sum_{s'\in\tilde S} \hat{P}_{s' \mid s,a}(s' \mid s,a)
  \left(R(s' \mid s,a) + \gamma^{\Delta t} \hat V_{k-1}(s')\right) \] is a function sampled from a Gaussian process
  with known priors and i.i.d Gaussian noise $\mathcal N(0,\sigma)$. 
If we allow Bayesian optimization for $Q_{s,k}(a)$ to sample
  $T$ actions for each state and run RTDP in Alg.~\ref{alg:rtdp} until it converges with respect to the Cauchy's convergence criterion~\cite{cauchy1821cours} with
  $\mathcal K <\infty$ iterations, in expectation, 
  \[\lim_{|\tilde S|,|D|\rightarrow \infty} |\hat V_{\mathcal K } (\cdot)- V^*(\cdot)|_{B_{\mathcal K}} \leq  \frac{\nu}{1-\gamma}\sqrt{\frac{2\eta_T}{T\log (1+\sigma^2)} },\]
  where $\eta_T$ is the
  maximum information gain of the selected actions~\cite[Theorem
  5]{srinivas2009gaussian},
  $\nu = \underset{s,t,k}{\max} \,\underset{a\in A}\min\,
  G_{s,t,k}(a)$, and $G_{s,t,k}(\cdot)$ is the acquisition
  function in \cite[Theorem 3.1]{wang2016est} for state $s\in \tilde S$, iteration $t=1,2,\cdots, T$ in Alg.~\ref{alg:bo} or~\ref{alg:bbo}, and iteration $k=1,2,\cdots, \mathcal K$ of the loop in Alg.~\ref{alg:rtdp}. 
\end{thm}
\begin{proof}
To prove Thm.~\ref{thm:ebd}, we first show the following facts:  (1) The state sampling procedure in Alg.~\ref{alg:statesample} stops in finite steps; (2) the difference between the value function computed by \ALGNAME and the optimal value function computed via asynchronous dynamic programming with an exact optimizer is bounded in expectation; 
(3) the optimal value function of the approximated MDP $\tilde{ \mathcal M}$ asymptotically converges to that of the original problem defined by the MDP $\mathcal M$.

\begin{claim} 
The expected number of iterations for the set of sampled states $\tilde S$ computed by Alg.~\ref{alg:statesample} to contain one state in $\mathcal G$ is finite.
\label{claim1}
\end{claim}

\begin{claimproof}
 Let $S_{good}$ be the set of states with non-zero probability to reach the goal region via finite actions.  Clearly, $\mathcal G\subset S_{good}$. Because there exists a state in the goal region that is reachable with finite steps following the policy $\pi^*$ starting from $s_0$, we have $s_0\in S_{good}$. Hence $\tilde S\cap S_{good}$ is non-empty in any iteration of \textsc{SampleInteriorStates} of Alg.~\ref{alg:statesample}. We can show that if the nearest state selected in Line~\ref{alg1uniform} of Alg.~\ref{alg:statesample} is in $S_{good}$, there is non-zero probability to extend another state in $S_{good}$ with the RRT procedure. To prove this, we first show for every state $s\in S_{good}\cap \tilde S$ there exists a set of actions with non-zero measure, in which each action $a$ satisfies $\int_{S_{good}}p_{s'\mid s,a}(s'\mid s,a) \dif s' >0$.
 
 
For every state $s \in S_{good}\cap \tilde S$, because $Q_{s}(a)$ is locally continuous at $a=\pi^*(s)$, for any positive real number $\zeta$, there exists a positive real number $\delta$ such that $\forall a\in \{a: g(a,\pi^*(s))<\delta, a\in A\}$, we have 
 \begin{align}
 |Q_s(a) - Q_{s}(\pi^*(s))| < \zeta.\label{eq:qspi}
 \end{align} 
 Notice that by the design of the reward function $R$, we have
 \begin{align}
Q_s(a) & =\int_{ S_{good}\cup(S\setminus S_{good})} p_{s' \mid s,a}(s' \mid s,a)
  \left(R(s'\mid s,a) + \gamma^{\Delta t} V^*(s')\right) \dif s'  \\
  &\leq \int_{S_{good}} p_{s' \mid s,a}(s' \mid s,a)
  \left(R(s'\mid s,a) + \gamma^{\Delta t} V^*(s')\right) \dif s' \nonumber \\
  & \;\; + \frac{C_a}{1-\gamma^{\mathcal T} } \int_{S\setminus S_{good}} p_{s' \mid s,a}(s' \mid s,a)
   \dif s' 
\end{align}
where $-C_a > 0$ is the smallest cost for either executing one action or colliding with obstacles. The inequality is because $\forall s'\in S\setminus S_{good}$,  
\begin{align}
R(s'\mid s,a) + \gamma^{\Delta t} V^*(s') \leq  \frac{C_a }{1-\gamma^{\mathcal T} } < 0.
\end{align} 

Because $\pi^*(s)=\argmax_{a\in A} Q_s(a)$ and the reward for the goal region is set large enough so that $Q_s(\pi^*(s)) >  0$, 
 there exists $r > 0$ such that $\forall q\in \R$ satisfying $q > Q_s(\pi^*(s))-r$, we have $q > 0$. 
 Let the arbitrary choice of $\zeta$ in Eq.~\eqref{eq:qspi} be $\zeta = r$. Because $Q_s(\pi^*(s)) - \zeta = Q_s(\pi^*(s)) - r < Q_s(a)$, we have  $$Q_s(a) >  0, \forall a\in \{a: g(a,\pi^*(s))<\delta, a\in A\},$$  and 
$$\int_{S_{good}} p_{s' \mid s,a}(s' \mid s,a)
  \left(R(s'\mid s,a) + \gamma^{\Delta t} V^*(s')\right) \dif s' > -\frac{C_a}{1-\gamma^{\mathcal T} } \int_{S\setminus S_{good}} p_{s' \mid s,a}(s' \mid s,a)
   \dif s' > 0$$
Hence $\int_{s'\in S_{good}} p_{s' \mid s,a}(s' \mid s,a)\dif s' > 0$ must hold for any action $a\in \{a: g(a,\pi^*(s))<\delta, a\in A\}$. Recall that one dimension of $a = (u, \Delta t)$ is the duration $\Delta t$ of the control $u$. Because the dynamics of the physics world is continuous, for any $a = (u, \Delta t')\in A$ such that $g((u,\Delta t),\pi^*(s))<\delta$ and $1 \leq\Delta t' \leq \Delta t$, we have $\int_{S_{good}}p_{s'\mid s,a}(s'\mid s,a) \dif s' >0$. 
Let  $A_s = \{(u,\Delta t'): g((u,\Delta t),\pi^*(s))<\delta, 1 \leq\Delta t' \leq \Delta t\}\cap A$. 

What remains to be shown is $\psi(A_s) > 0$. Because there exists an open set $A^o$ such that $A$ is the closure of $A^o$, $\pi^*(s)$ is either in $A^o$ or a limit point of $A^o$. If $\pi^*(s)$ is in the open set $A^o$, there exist $0<\delta' \leq \delta$ such that $\{a: g(a,\pi^*(s))<\delta'\} \subset A$, and so we also have $\psi(A_s) > 0$. If $\pi^*(s)$ is a limit point of the open set $A^o$, there exist $a'\in A^o$ such that $g(a', \pi^*(s)) < \delta/2$ and $a'\neq \pi^*(s)$. Because $a'\in A^o$, there exist $0< \delta' \leq \delta/2$ such that $\{a:g(a', a) < \delta'\}\subset A$. For any $a \in \{a:g(a', a) < \delta'\}$, we have $g(a, \pi^*(s)) \leq g(a,a') + g(a',\pi^*(s)) < \delta' + \delta/2 \leq \delta$. Hence $\{a:g(a', a) < \delta'\} \subset A_s$, and so $\psi(A_s) > 0$. Thus for any $\pi^*(s)\in A$, we have $\psi(A_s) > 0$.

So, for every state $s\in S_{good}\cap \tilde S$ and action $a\in A_s$ with $\psi(A_s) > 0$,  $\int_{S_{good}}p_{s'\mid s,a}(s'\mid s,a) \dif s' >0$. 
As a corollary, $\rho(\{s':p(s'\mid s,a)>0\}\cap S_{good}) > 0$ holds $\forall s\in S_{good}\cap \tilde S, a\in A_s$. 
\hide{
 Because the number of actions taken to get to the goal region is finite, there exists a longest path with $K$ states under policy $\pi^*$: $s_0,s_1,s_2,\cdots,s_K$ with $s_K\in \mathcal G$. If there exists a longer path, we can simply replace the selected path with the longer one. $\forall k=0,1,\cdots,K-1$, since $Q_{s_k}(a)$ is locally continuous at $a=\pi^*(s)$, for any positive real number $\zeta$, there exists a positive real number $\delta$ such that $\forall a\in A, s\in S$ satisfying $g(a,\pi^*(s))<\delta$ and $d(s_k,s)<\delta$, we have $|Q_s(a) - Q_{s_k}(\pi^*(s_k))| < \zeta$. Let $F_k$ denote the set of states which requires at most $K-k$ actions to reach the goal region with probability at least $\theta$ ($\theta$ is chosen to ensure $s_k\in F_k$), $k=0,\cdots, K$. We show that $\rho(F_k)>0, \forall k$ by induction. When $k=K-1$, $\forall a\in A, s\in S$ satisfying $g(a,\pi^*(s))<\delta$ and $d(s_{K-1},s)<\delta$, we have 
\begin{align}
Q_s(a) & =\int_{s'\in S\setminus \mathcal G} p_{s' \mid s,a}(s' \mid s,a)
  \left(R(s'\mid s,a) + \gamma^{\Delta t} V^\pi(s')\right) \dif s' + \int_{s'\in \mathcal G} p_{s' \mid s,a}(s' \mid s,a) R_{goal} \dif s' 
\end{align}
where $R_{goal} > 0$ is the reward for going to the goal region. Because the longest path to the goal is $K$ steps, if the next state $s'$ given the current state $s$ and action $a$ is not in the goal region, it will not be able to reach the goal forever. Hence, the best policy for it is to stay off the obstacles and incur the action cost $C_a <0$:
\begin{align}
Q_s(a) & = \underbrace{\int_{s'\in S\setminus \mathcal G} p_{s' \mid s,a}(s' \mid s,a)
  \left( C_a + \frac{C_a \gamma^{\Delta t}}{1-\gamma^{\max \Delta t}} \right) \dif s' }_{Q'_s(a)}+ \underbrace{\int_{s'\in \mathcal G} p_{s' \mid s,a}(s' \mid s,a) R_{goal} \dif s' }_{Q''_s(a)}
\end{align}
Because $|Q_s(a) - Q_{s_{K-1}}(\pi^*(s_{K-1}))| < \zeta$, we have 
\begin{align}
Q_s(a) &> Q_{s_{K-1}}(\pi^*(s_{K-1})) -\zeta \\
Q'_s(a)+Q''_s(a) &> Q'_{s_{K-1}}(\pi^*(s_{K-1})) + Q''_{s_{K-1}}(\pi^*(s_{K-1})) -\zeta
\end{align}
If $Q''_s(a)>0$, let $\theta \leq \int_{s'\in \mathcal G} p_{s' \mid s,a}(s' \mid s,a) \dif s', \forall s,a$, and we have $\{s \mid d(s_{K-1},s)<\delta\}\subseteq F_{K-1}$, $\rho(F_{K-1})>0$. If $Q''_s(a)=0$, because all the probability mass is shifted to non goal region, we must have $Q'_s(a) < Q'_{s_{K-1}}(\pi^*(s_{K-1}))$; however, $\zeta$ can be chosen small enough such that $Q''_s(a)> Q'_{s_{K-1}}(\pi^*(s_{K-1})) + Q''_{s_{K-1}}(\pi^*(s_{K-1})) -\zeta - Q'_s(a) > 0$, which creates a contradiction. In all, $\rho(F_{K-1})>0$ and $\int_{s'\in \mathcal G} p_{s' \mid s,a}(s' \mid s,a) \dif s' > 0$.

Assume $\rho(F_{k+1}) > 0$. $\forall a\in A, s\in S$ satisfying $g(a,\pi^*(s))<\delta$ and $d(s_{K-1},s)<\delta$, we have 
\begin{align}
Q_s(a) & = \int_{s'\in S\setminus \mathcal F_{k+1}} p_{s' \mid s,a}(s' \mid s,a)
  \left(R(s'\mid s,a) + \gamma^{\Delta t} V^\pi(s')\right) \dif s' \nonumber \\
  &  + \int_{s'\in \mathcal F_{k+1}} p_{s' \mid s,a}(s' \mid s,a)  \left(R(s'\mid s,a) + \gamma^{\Delta t} V^\pi(s')\right) \dif s' 
\end{align}
Similar to the arguments made when $k=K-1$, we can get $\rho(F_{k}) > 0$ and $\int_{s'\in \mathcal F_{k+1}} p_{s' \mid s,a}(s' \mid s,a) \dif s' > 0$.
} 

Now we can show that there is non-zero probability to extend a state on the near-optimal path in $S_{good}$ with the RRT procedure in one iteration of \textsc{SampleInteriorStates} in Alg.~\ref{alg:statesample}.
Let $\theta = \min_{s\in \tilde S\cap S_{good}, a\in A_s} \rho(\{s':p(s'\mid s,a)>0\}\cap S_{good}) > 0$. 
With the finite $\tilde S$ for some iteration, we can construct a Voronoi diagram based on the vertices from the current set of sampled states $\tilde S$. $\forall  s\in S_{good}\cap \tilde S$, there exists a Voronoi region $\Vor(s)$ associated with state $s$. We can partition this Voronoi region  $\Vor(s)$ to one part, $\Vor(A_{s})\subset \Vor(s)$, containing states in $S_{good}$ generated by actions in $A_{s}$ and its complement, $\Vor(A\setminus A_{s}) = \Vor(s) \setminus \Vor(A_s)$. 
Notice that $A_s$ includes actions with the minimum duration, and the unit for the minimum duration can be set small enough so that $\rho(\{s':p(s'\mid s,a)>0, a\in A_s, s'\in S_{good}\}\cap \Vor(s) ) > 0$\footnote{This is because $\Vor(s)$ is a neighborhood of $s$, and there exists an action $a\in A_s$ such that a next state $s'\in S_{good}$ is in the interior of $\Vor(s)$ given the current state $s$. So there exists a small ball in $S$ with $s'$ as the center such that this ball is a subset of both $\Vor(s)$ and $\{s':p(s'\mid s,a)>0, a\in A_s, s'\in S_{good}\}$ (by the continuity of $p_{s'\mid s, a}$). }. Since  $\{s':p(s'\mid s,a)>0, a\in A_s\}\cap S_{good}\cap \Vor(s)\subset \Vor(A_{s})$, we have $\rho(\Vor(A_{s})) > 0, \forall s\in \tilde S$. We denote $p_s=\min_s \frac{\rho(\Vor(A_{s}))}{\rho(S)} > 0$ and $p_a=\min_s \frac{\psi(A_{s})}{\psi(A)} > 0$. With probability at least $p_s$, there is a random state sampled in $\Vor(A_{s})$ in this iteration. With probability at least $p_a$, at least an action in $A_s$ is selected to test distance, and with probability at least $\theta$, a state in $S_{good}$ can be sampled from the transition model conditioned on the state $s$ and the selected action in $A_s$. 
 
 Next we show that \textsc{SampleInteriorStates} in Alg.~\ref{alg:statesample} constructs an RRT whose finite set of sampled states $\tilde S$ contains at least one goal state in expectation.
 
By assumption, the goal state is reachable with finite actions. For any $s\in S_{good} \cap \tilde S$, the goal region is reachable from $s$ in finite steps. Notice that once a new state in $\{s':p(s'\mid s,a)>0,a\in A_s\}\cap S_{good}$ is sampled, $s'$ uses one less step than $s$ to reach the goal region. Let $K$ be the largest finite number of actions necessary to reach the goal region $\mathcal G$ from the initial state $s_0$. Hence, with at most a finite number of $\frac{K}{\theta p_s p_a}$ iterations in expectation (including both loops for sampling actions and loops for sampling states), at least a goal state will be added to $\tilde S$.

\hfill Q.E.D.
\end{claimproof}
\\

\begin{claim}
  Let $\tilde V^*$ be the optimal value function computed via asynchronous dynamic programming with an exact optimizer. If Alg.~\ref{alg:rtdp} converges with $\mathcal K<\infty$ iterations,  $$ |\hat V_{\mathcal K } (\cdot)- \tilde V^*(\cdot)|_{B_{\mathcal K}} \leq  \frac{\nu}{1-\gamma}\sqrt{\frac{2\eta_T}{T\log (1+\sigma^2)} }.$$
\end{claim}
 \begin{claimproof}
The RTDP process of \ALGNAME in Alg.~\ref{alg:rtdp} searches for the relevant set $B_{\mathcal K}$ of \ALGNAME's policy $\hat \pi$ via recursion on stochastic paths (trials). If \ALGNAME converges, all states in $B_{\hat \pi}$ should have been visited and their values $\hat V(s), \forall s\in B_{\mathcal K}$ have converged. Compared to asynchronous dynamic programming (ADP) with an exact optimizer, our RTDP process introduces small errors at each trial, but eventually the difference between the optimal value function computed by ADP and the value function computed by \ALGNAME is bounded.

In the following, the order of states to be updated in ADP is set to follow RTDP. This order does not matter for the convergence of ADP as any state in $B_{\hat \pi}$ will eventually be visited infinitely often if $\mathcal K\rightarrow \infty$~\cite{sutton1998reinforcement}. 
We denote the value for the $i$-th state updated at iteration $k$ of RTDP to be $\hat V_{ki}$, the corresponding value function updated by RTDP with an exact optimizer only for this update to be $\hat V^*_{ki}$, and the difference between them to be $\epsilon_{ki} = |\hat V^*_{ki} - \hat V_{ki}|$. 

 According to \cite[Theorem
3.1]{wang2016est}, in expectation, for any $i$-th state $s_{ki}$ to be updated at iteration $k$, the following inequality holds:
\begin{align*} \label{eq:vk}
\epsilon_{ki} \leq \nu_{ki} \sqrt{\frac{2 \eta_T}{T\log (1+\sigma^2)} },
\end{align*} 
where $$\nu_{ki} = \underset{t\in[1,T]}{\max}\min_{a\in A} G_{s_i,t,k}(a),$$  and $$G_{s_i,t,\cdot}(a) = \frac{h_u(s_i) - \mu_{t-1}(a)}{\sigma_{t-1}(a)}$$ is the acquisition function in \cite[Theorem 3.1]{wang2016est}, which makes use of the assumed upper bound $h_u(\cdot)$ on the value function. 
Let the sequence of states to be updated at iteration $k$ be $s_{k1},s_{k2},\cdots,s_{kn_k}$ and $\nu = \max_{k\in[1, \mathcal K ],i\in[1,n_k]} \nu_{ki}$. For any iteration $k$ and state $s_{ki}$, we have 
$$\epsilon_{ki} \leq \nu \sqrt{\frac{2 \eta_T}{T\log (1+\sigma^2)} } = \epsilon.$$
So our optimization introduces error of at most $\epsilon$ to the
optimization of the Bellman equation at any iteration. Furthermore, we can bound the difference between the value for the $i$-th state updated at iteration $k$ of RTDP ($\hat V_{ki}$) and the corresponding value function updated by ADP ($\tilde V^*_{ki}$). More specifically, the following inequalities hold for any $\mathcal K=1,2,\cdots, \infty$: 
\begin{align*}
|\hat V_{11} - \tilde V^*_{11}| &\leq \epsilon,\\
|\hat V_{12} - \tilde V^*_{12}| &\leq \epsilon + \gamma \epsilon,\\
\cdots &,\\
|\hat V_{1n_1} - \tilde V^*_{1n_1}| &\leq \sum_{i=1}^{n_1} \gamma^{i-1}\epsilon,\\ 
\cdots &,\\
\cdots &,\\
|\hat V_{\mathcal K1} - \tilde V^*_{\mathcal K1}| &\leq \epsilon + \gamma\sum_{i=1}^{n_1+\cdots +n_{\mathcal K -1}} \gamma^{i-1}\epsilon,\\
|\hat V_{\mathcal K2} - \tilde V^*_{\mathcal K2}| &\leq \epsilon + \gamma \epsilon + \gamma^2\sum_{i=1}^{n_1+\cdots +n_{\mathcal K -1}} \gamma^{i-1}\epsilon,\\
\cdots &,\\
|\hat V_{\mathcal K n_{\mathcal K}} - \tilde V^*_{\mathcal Kn_{\mathcal K}}| &\leq \sum_{i=1}^{n_1+\cdots +n_{\mathcal K}} \gamma^{i-1}\epsilon  < \frac{\epsilon}{1-\gamma} = \frac{\nu}{1-\gamma}\sqrt{\frac{2\eta_T}{T\log (1+\sigma^2)} }.
\end{align*}

Notice that $\hat V_{ki}$ converges because it monotonically decreases for any state index $i$ and it is lower bounded by $\left(\frac{C }{1-\gamma^{\mathcal T} }\right)$ where $C<0$ is the highest cost, e.g. colliding with obstacles. Since we use Cauchy's convergence test~\cite{cauchy1821cours}, we can set the threshold for the convergence test to be negligible. Hence for some $\mathcal K<\infty$, both $\hat V_k$ and $\tilde V_k$ converge according to Cauchy's convergence test, and we have 
$$ |\hat V_{\mathcal K } (\cdot)- \tilde V^*(\cdot)|_{B_{\mathcal K}} \leq  \frac{\nu}{1-\gamma}\sqrt{\frac{2\eta_T}{T\log (1+\sigma^2)} }.$$
\hfill Q.E.D.
\end{claimproof}
\\

\begin{claim} $\tilde V^*$, the optimal value function of the approximated MDP $\tilde{ \mathcal M}$, asymptotically converges to $V^*$, the optimal value function of the original problem defined by the MDP $\mathcal M$.
\end{claim}
\begin{claimproof}
We consider the asymptotic case where the size of the dataset $|D|\rightarrow \infty$ and the number of states sampled $|\tilde S|\rightarrow \infty$. Notice that these two limit does not contradict Claim 1.1 because \ALGNAME operates in a loop and we can iteratively sample more states by calling \textsc{SampleStates} in Alg.~\ref{alg:boidp}. When $|D|\rightarrow \infty$, $p_{s'|s,a}$ converges to the true transition model.

Because the states are sampled uniformly randomly from the state space $S$ in Line~\ref{alg1uniform} of Alg.~\ref{alg:statesample}, when $|\tilde S|\rightarrow \infty$, the set $\tilde S$ can be viewed as uniform random samples from the reachable state space\footnote{Alg.~\ref{alg:statesample} does not necessarily lead to uniform samples of states in the state space. However, as the number of states sampled approaches infinity, $|\tilde S|\rightarrow \infty$, we can construct a set of finite and arbitrarily small open balls that cover the (reachable) state space such that there exists at least one sampled state in any of those balls. Such a cover exists because the state space is compact. If the samples are not uniform, we can simply adopt a uniform sampler on top of Alg.~\ref{alg:statesample}, and throw away a fixed proportion of states so that the remaining set of states are uniform samples from the state space $S$.}. So the value function for the optimal policy of $\tilde{\mathcal M}$ asymptotically converges to that of $\mathcal M$:
\begin{align*}
\lim_{|\tilde S|,|D|\rightarrow \infty } |\tilde V^* - V^*|_{\infty} = 0.
\end{align*}
 \hfill Q.E.D.
\end{claimproof}

Thm.~\ref{thm:ebd} directly follows Claim 1.1, 1.2, and 1.3. By the triangle inequality of $L_\infty$, we have
\begin{align*}
\lim_{|\tilde S|,|D|\rightarrow \infty } |\hat V_{\mathcal K} - V^*|_{B_{\mathcal K}}  &\leq \lim_{|\tilde S|,|D|\rightarrow \infty } |\hat V_{\mathcal K} - \tilde V^*|_{B_{\mathcal K}} + \lim_{|\tilde S|,|D|\rightarrow \infty }  |\tilde V^* - V^*|_{B_{\mathcal K}} \\
& \leq \frac{\nu}{1-\gamma}\sqrt{\frac{2\eta_T}{T\log (1+\sigma^2)} }
\end{align*}
holds in expectation.
\end{proof}
\section{Implementation and Experiments}
\label{sec:exp}

We tested our approach in a quasi-static problem, in which a robot
pushes a circular object through a planar workspace with obstacles in
simulation\footnote{All experiments were run with Python 2.7.6 on
  Intel(R) Xeon(R) CPU E5-2680 v3 @ 2.50GHz with 64GB memory.}. We
represent the action by the robot's initial relative position $x$ to
the object (its distance to the object center is fixed), the direction
of the push $z$, and the duration of the push $\Delta t$, which are
illustrated in Fig.~\ref{fig:push}.  The companion video shows the
behavior of this robot, controlled by a policy derived by \ALGNAME
from a set of training examples.
%

In this problem, the basic underlying dynamics in free space with no
obstacles are location invariant; that is, that the change in state
$\Delta s$ resulting from taking action $a=(u,\Delta t)$ is independent of the
state $s$ in which $a$ was executed. 
We are given a training dataset $D = \{\Delta s_i, a_i\}_{i=0}^N$, where
$a_i$ is an action and $\Delta s_i$ is the resulting state change, 
collected in the free space in a simulator. Given a new query for action $a$, we predict the distribution of
$\Delta s$ by looking at the subset
$D' = \{\Delta s_j, a_j\}_{j=0}^M \subseteq D$ whose actions $a_j$ are
the most similar to $a$ (in our experiments we use 1-norm distance to
measure similarity), and fit a Gaussian mixture model on
$\Delta s_j$ using the EM algorithm, yielding
an estimated continuous state-action transition model
$p_{s'\mid s,a}(s + \Delta_s\mid s,a) = p_{\Delta s \mid a}(\Delta s \mid a).$
We use the Bayesian information criterion (BIC) to determine the
number of mixture components.
\begin{figure}[t]
        \centering
        \includegraphics[width=0.6\textwidth]{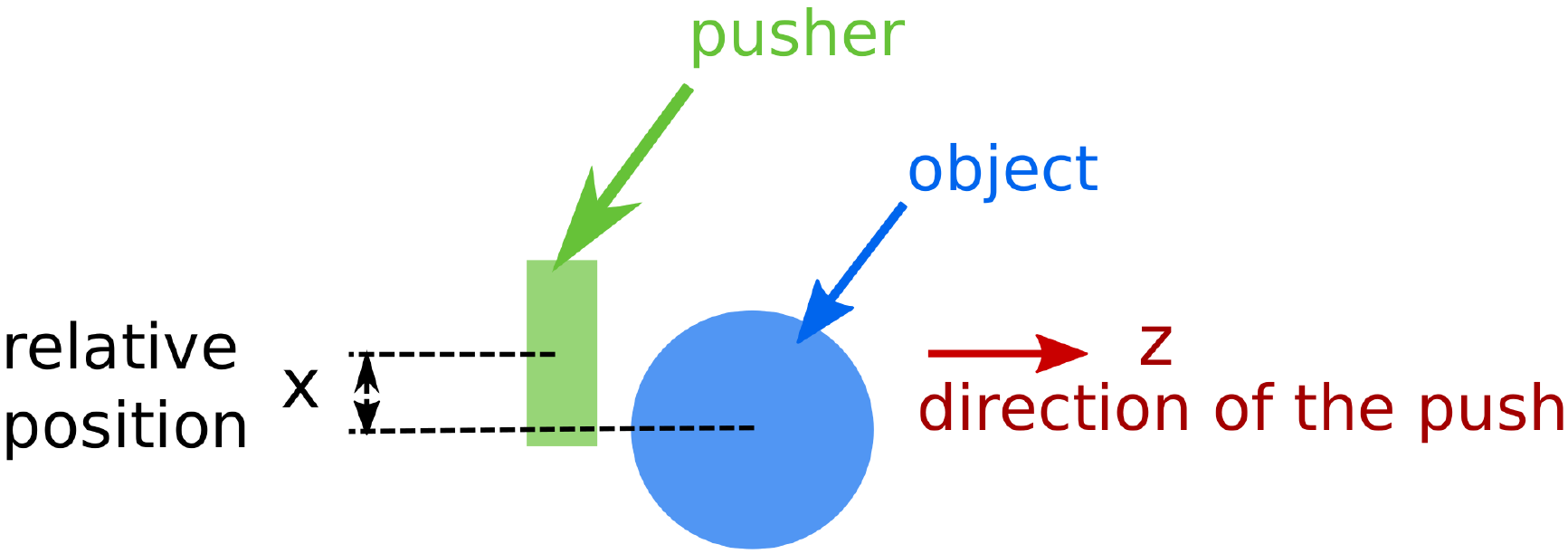}
        \caption{Pushing a circular object with a rectangle pusher. }
        \label{fig:push}
\end{figure}


\subsection{Importance of learning accurate models}
\label{ssec:learning}

Our method was designed to be appropriate for use in systems whose dynamics are not well modeled with uni-modal Gaussian noise. The experiments in this section explore the question of
whether a uni-modal model could work just as well, using a simple
domain with known dynamics $s' = s + T(a)\rho$, where the relative
position $x=0$ and duration $\Delta t=1$ are fixed,  
the action is the direction of motion, $a=z\in[0,2\pi)$, $T(a)$ is the
rotation matrix for angle,\hide{$\begin{bmatrix} 
\cos(a) & \sin(a)\\
-\sin(a) & \cos(a)
\end{bmatrix}$,} and the noise is

\begin{small}
\begin{align*}
\rho \sim 0.6\mathcal N(\begin{bmatrix}
5.0\\
5.0
\end{bmatrix},\begin{bmatrix}
2.0 &0.0\\
0.0 & 2.0
\end{bmatrix})+0.4\mathcal N(\begin{bmatrix}
5.0\\
-5.0
\end{bmatrix},\begin{bmatrix}
2.0 & 0.0\\
0.0 & 2.0
\end{bmatrix}).
\end{align*}
\end{small}
 \begin{figure}[t]
        \centering
        \includegraphics[width=0.8\textwidth]{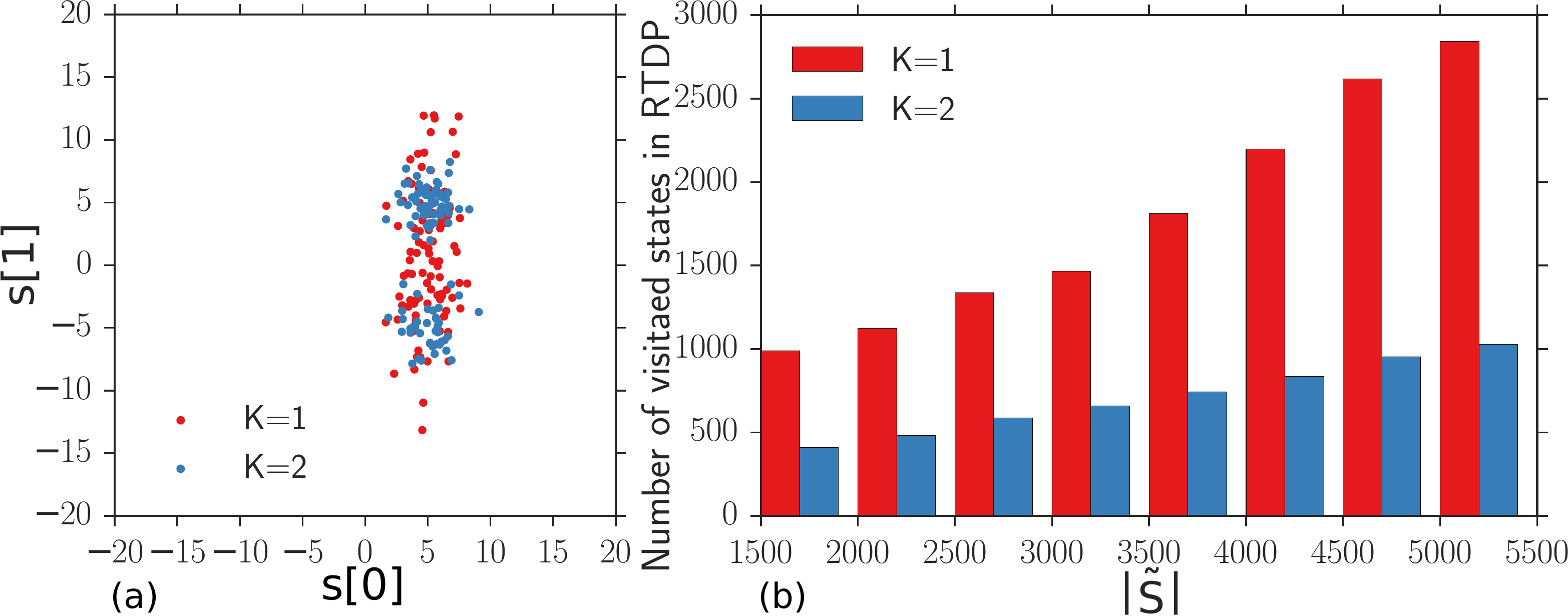}
        \caption{(a) Samples from the single-mode Gaussian  transition model ($K=1$) and the two-component  Gaussian mixture transition model ($K=2$) in the free space when $a=0$. (b) The number of visited states (y-axis) increases with the number of sampled states $|\tilde S|$ (x-axis).  Planning with $K=2$ visits fewer states in RTDP than with $K=1$.} 
        \label{fig:toysample}
\end{figure}

\begin{figure}[t]
\centering
        \includegraphics[width=0.7\textwidth]{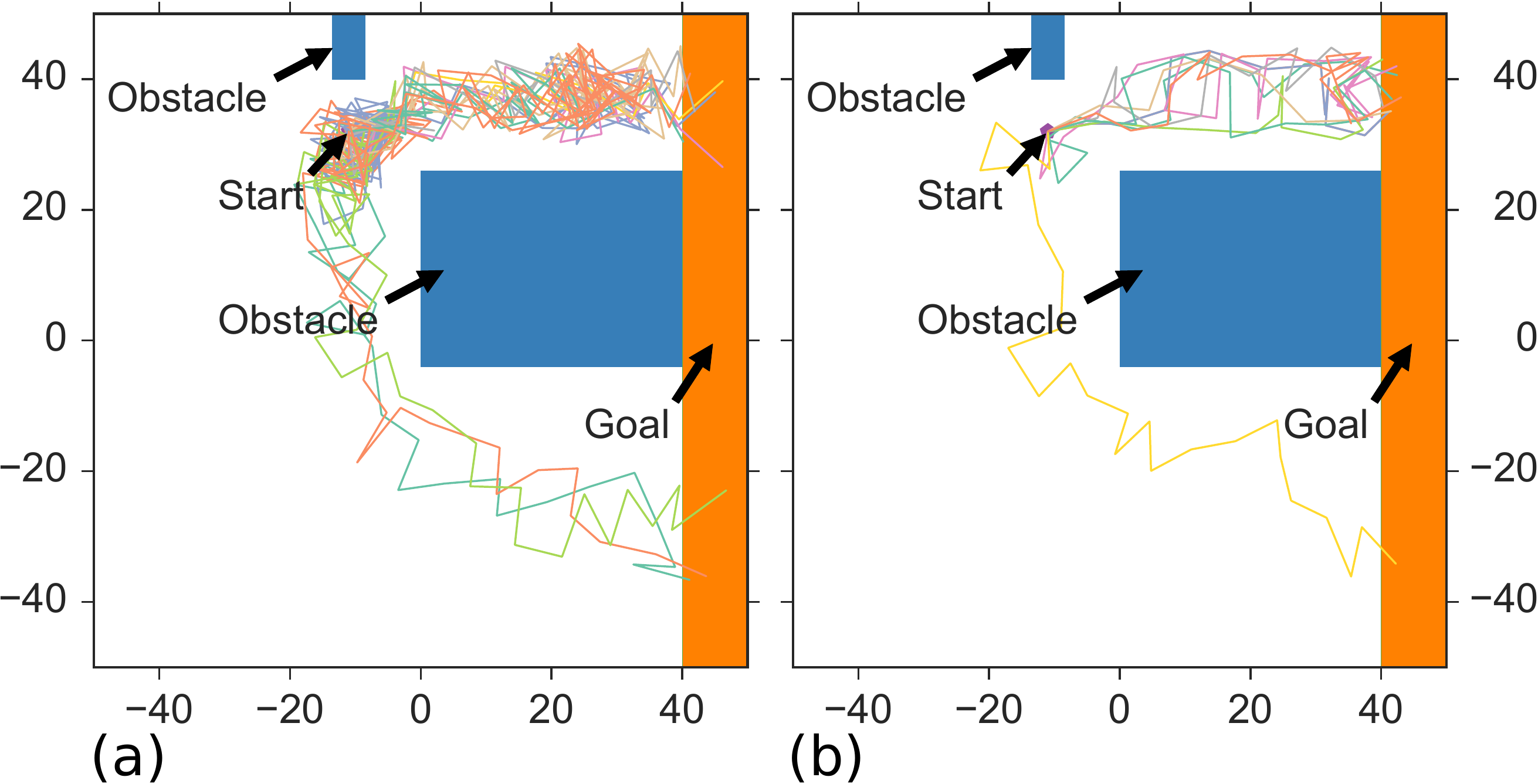}
        \caption{(a) Samples of 10 trajectories with $K=1$. (b) Samples of 10 trajectories  with $K=2$.  Using the
correct number of components for the transition model improves the quality of the trajectories.} 
        \label{fig:toytraj}
\end{figure}

\begin{figure}[t]
        \centering
        \includegraphics[width=0.8\textwidth]{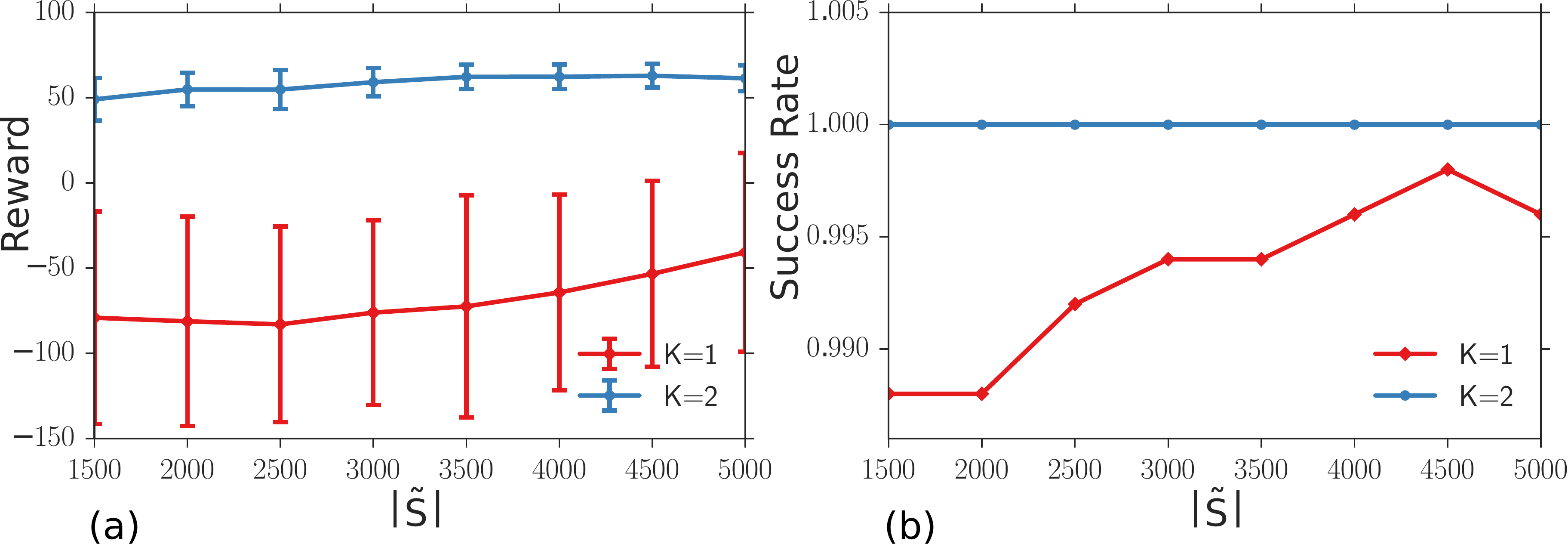}
        \caption{(a): Reward. (b): Success rate. Using two components ($K=2$) performs much better
          than using one component ($K=1$) in terms of reward and success rate.} 
        \label{fig:toyreward}
\end{figure}

We sample $\rho$ from its true distribution and fit a Gaussian ($K=1$)
and a mixture of Gaussians ($K=2$). The samples from $K=1$ and $K=2$
are shown in Fig.~\ref{fig:toysample}~(a).  We plan with both models where
each action has an instantaneous reward of $-1$, hitting an obstacle
has a reward of $-10$, and the goal region has a reward of $100$. The
discount factor $\gamma = 0.99$. To show that the results are
consistent, we use Algorithm~\ref{alg:statesample} to sample 1500 to
5000 states to construct $\tilde S$, and plan with each of them using
$100$ uniformly discretized actions within $1000$ iterations of RTDP. 

To compute the Monte Carlo reward, we simulated 500 trajectories for
each computed policy with the true model dynamics, and for each
simulation, at most 500 steps are allowed. We show 10 samples of
trajectories for both 
$K=1$ and $K=2$ with $|\tilde S|=5000$, in Fig~\ref{fig:toytraj}.  Planning with the right model $K=2$ 
          tends to find better trajectories, while because $K=1$ puts density on many states that the
          true model does not reach, the policy of $K=1$ in Fig~\ref{fig:toytraj}~(a) causes the
          robot to do extra maneuvers or even choose a longer trajectory to avoid obstacles that it
         actually has very low probability of hitting. 
           As a result, the reward and success rate for $K=2$ are both higher than $K=1$, as shown in Fig.~\ref{fig:toyreward}.
Furthermore, because the single-mode Gaussian
estimates the noise to have a large variance, it causes RTDP to visit
many more states than necessary, as shown in
Fig.~\ref{fig:toysample}~(b). 




\subsection{Focusing on the good actions and states}
In this section we demonstrate the effectiveness of our strategies for
limiting the number of states visited and actions modeled.
We denote using Bayesian
optimization in Lines~\ref{opt1} 
and~\ref{opt2} in Algorithm~\ref{alg:rtdp} as BO and using random selections as
Rand. 
\begin{figure}[h]
        \centering
        \includegraphics[width=0.6\textwidth]{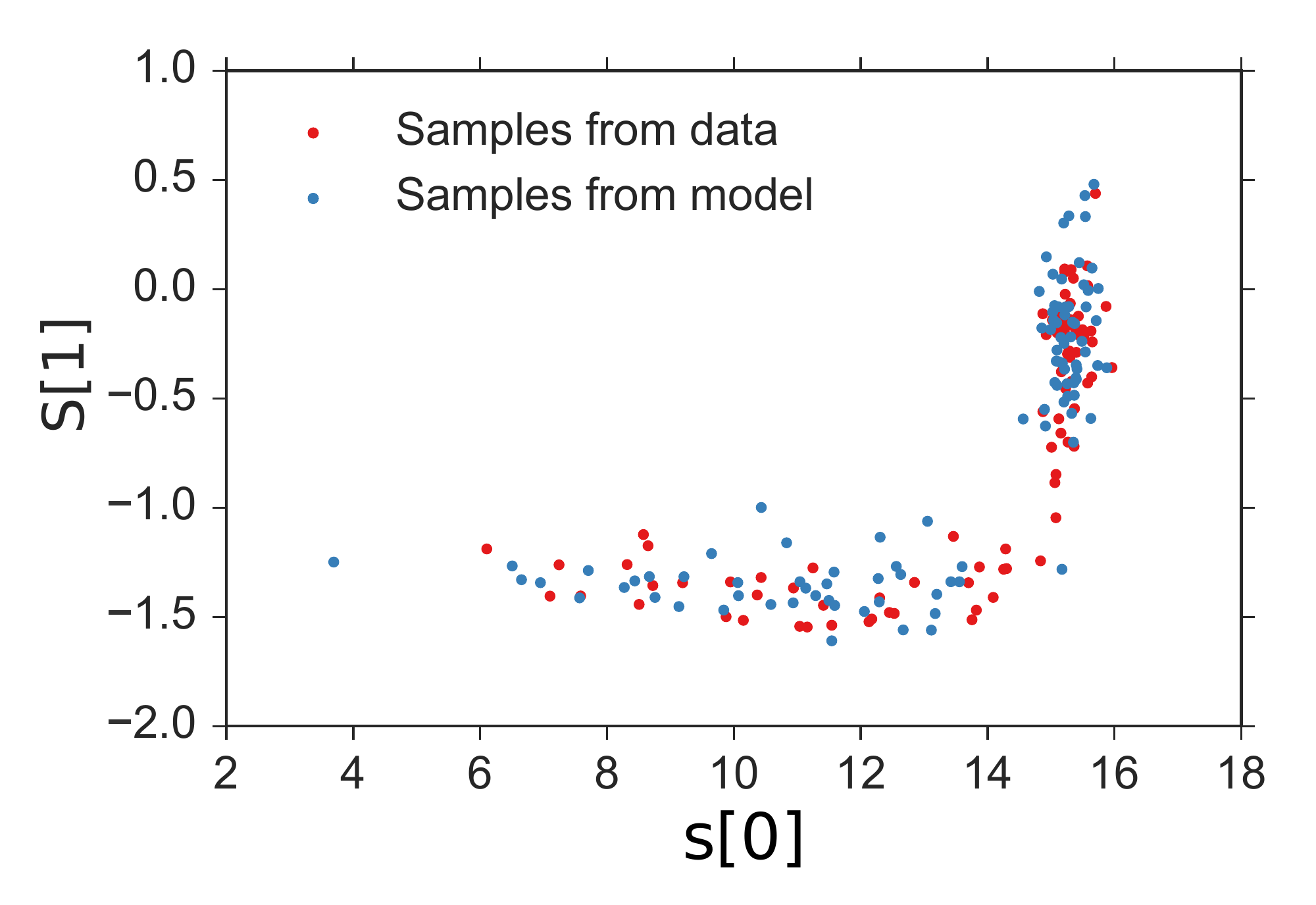}
        \caption{The conditional distribution of $\Delta s$ given
          $a=(z,x,\Delta t)=(0.0,0.3,2.0)$ is a multi-modal Gaussian. 
          }  
        \label{fig:pushmodel}
\end{figure}

We first demonstrate why BO is better than random for optimizing
$Q_s(a)$ with the simple example from Sec.~\ref{ssec:learning}. We plot the $Q_s(a)$ in the first iteration of RTDP where $s=[-4.3,33.8]$, and let random and BO in Algorithm~\ref{alg:bo} each pick 10 actions to
evaluate sequentially as shown in Fig~\ref{fig:boaction}~(a). We use the GP implementation and the default
Matern52 kernel implemented in the GPy module~\cite{gpy2014} and
optimize its kernel parameters every 5 selections. The first point for
both BO and Rand is fixed to be $a=0.0$. We observe that BO is able to
focus its action selections in the high-value region, and BO is also
able to explore informative actions if it has not found a good value
or if it has finished exploiting a good region (see selection
10). Random action selection wastes choices on regions that have
already been determined to be bad.
\begin{figure}[h]
\centering
\includegraphics[width=0.8\textwidth]{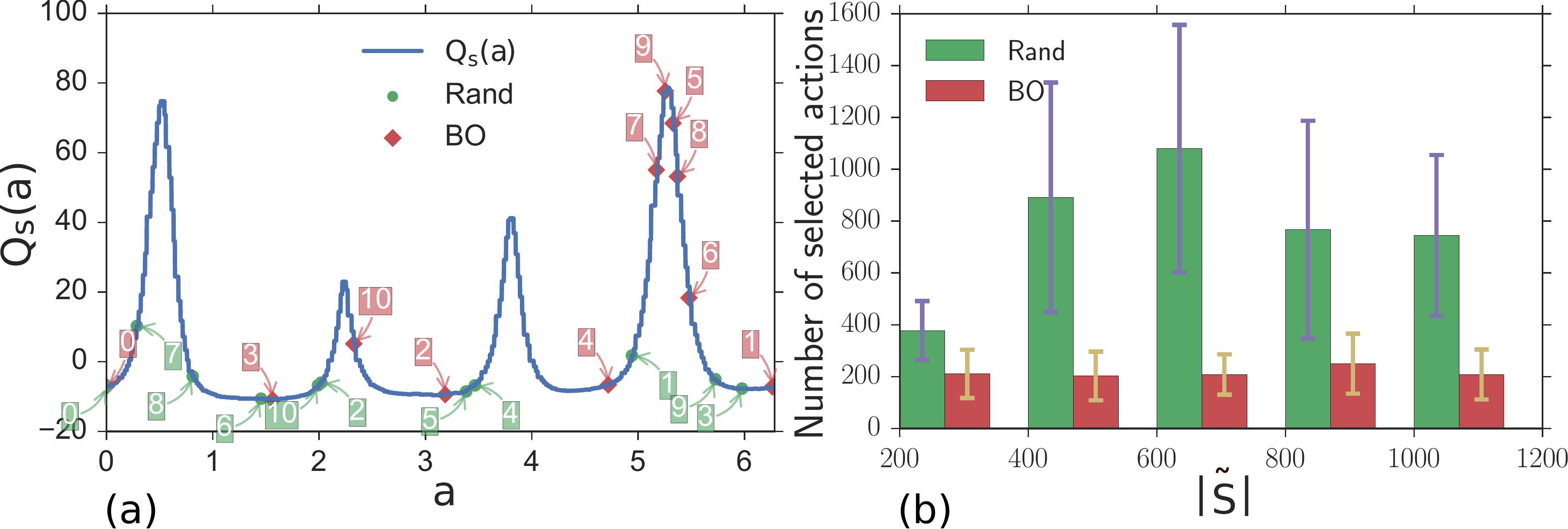}
        \caption{(a) We optimize $Q_s(a)$ with BO and Rand by sequentially sampling 10 actions. BO selects actions more strategically than Rand. (b) BO samples fewer actions than Rand in the pushing problem for all settings of $|\tilde S|$.}
        \label{fig:boaction}
\end{figure}

\begin{figure}[h]
\centering
\includegraphics[width=0.8\textwidth]{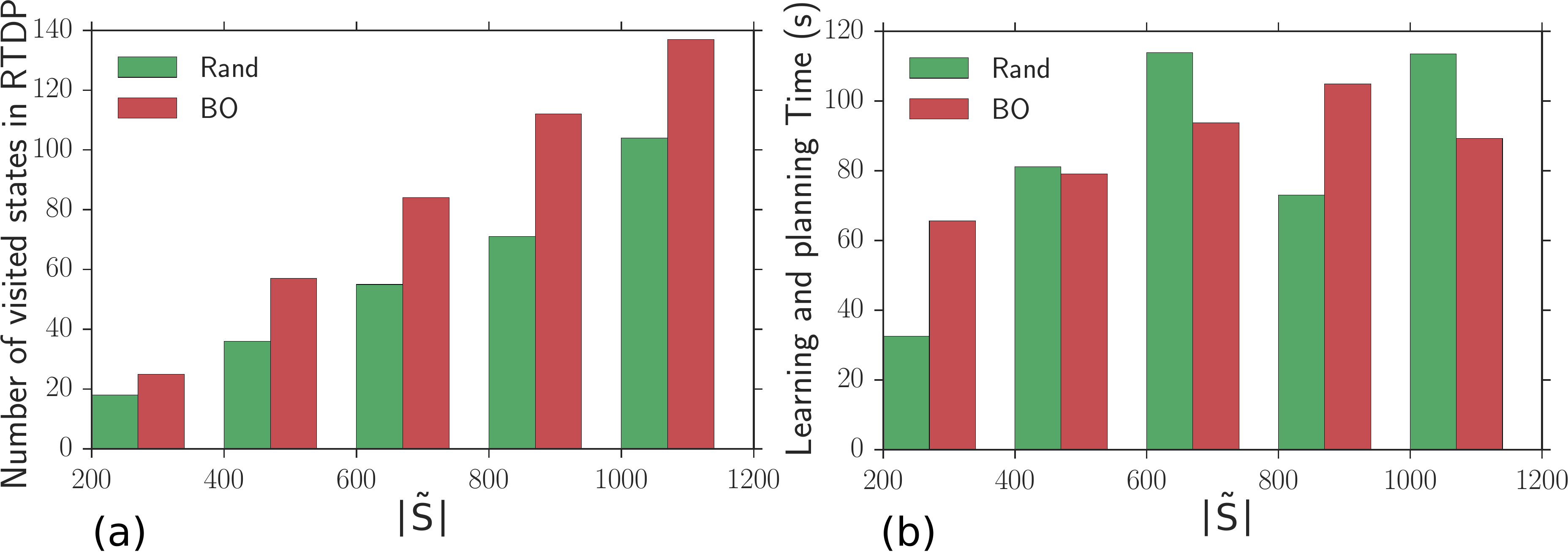}
        \caption{ (a) Number of visited states in RTDP.  Both of Rand and BO consistently focus on about 10\% states for planning. (b) Learning and planning time of BO and Rand.} 
        \label{fig:pushtime}
\end{figure}
\begin{figure}[h]
        \centering
        \includegraphics[width=0.8\textwidth]{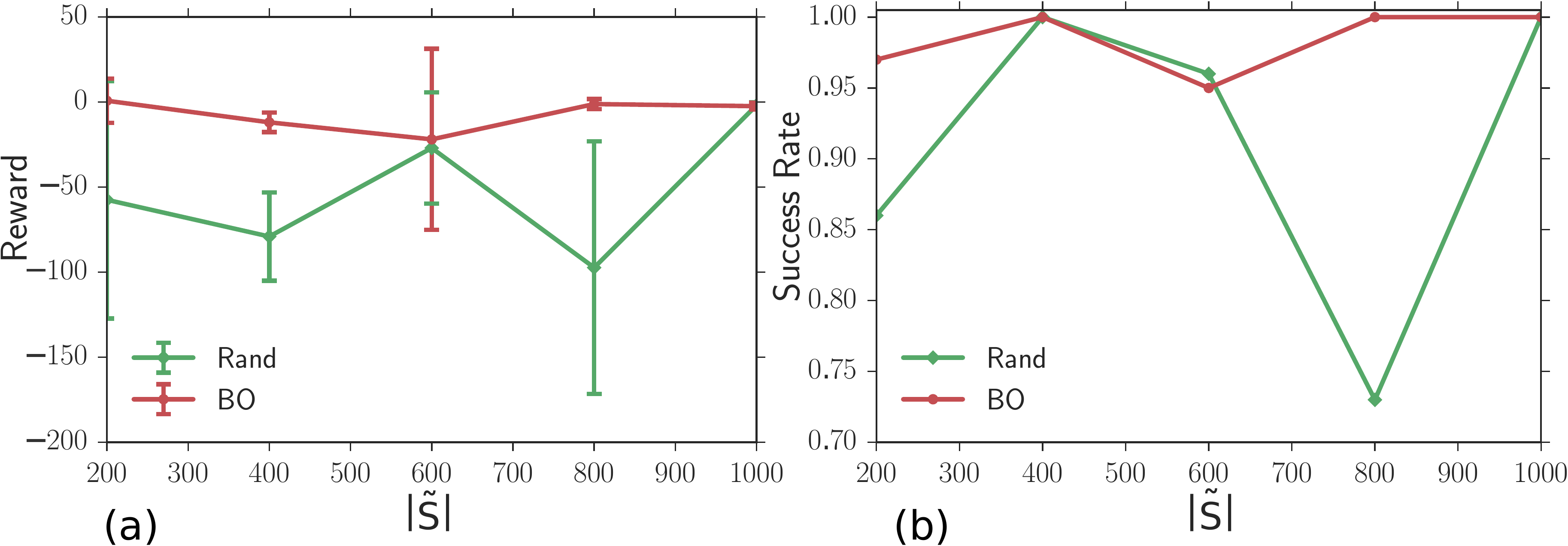}
        \caption{(a) Reward. (b) Success rate. BO achieves better
          reward and success rate, with many fewer actions and slightly more visited states.}
        \label{fig:pushreward} 
\end{figure}
Next we consider a more complicated problem in which the action is the high
level control of a pushing problem $a=(z, x, \Delta t)$,
$z\in[0,2\pi], x\in[-1.0,1.0], \Delta t\in[0.0,3.0]$ as illustrated in
Fig.~\ref{fig:push}. 
The instantaneous reward is $-1$ for each free-space motion,
$-10$ for hitting an obstacle, and $100$ 
for reaching the goal; $\gamma=0.99$.
We collected $1.2\times 10^6$ data points of the form $(a, \Delta
  s)$ with $x$ and $\Delta t$ as variables in the Box2D simulator~\cite{box2d}  where noise
  comes from variability of the executed action. We make use of the fact
  that the object is cylindrical (with radius 1.0) to reuse data. 
  An example of the
  distribution of $\Delta s $ given $a=(0.0,0.3,2.0)$ is shown
  in Fig.~\ref{fig:pushmodel}. 
\begin{figure}[h]
\centering
\includegraphics[width=0.7\textwidth]{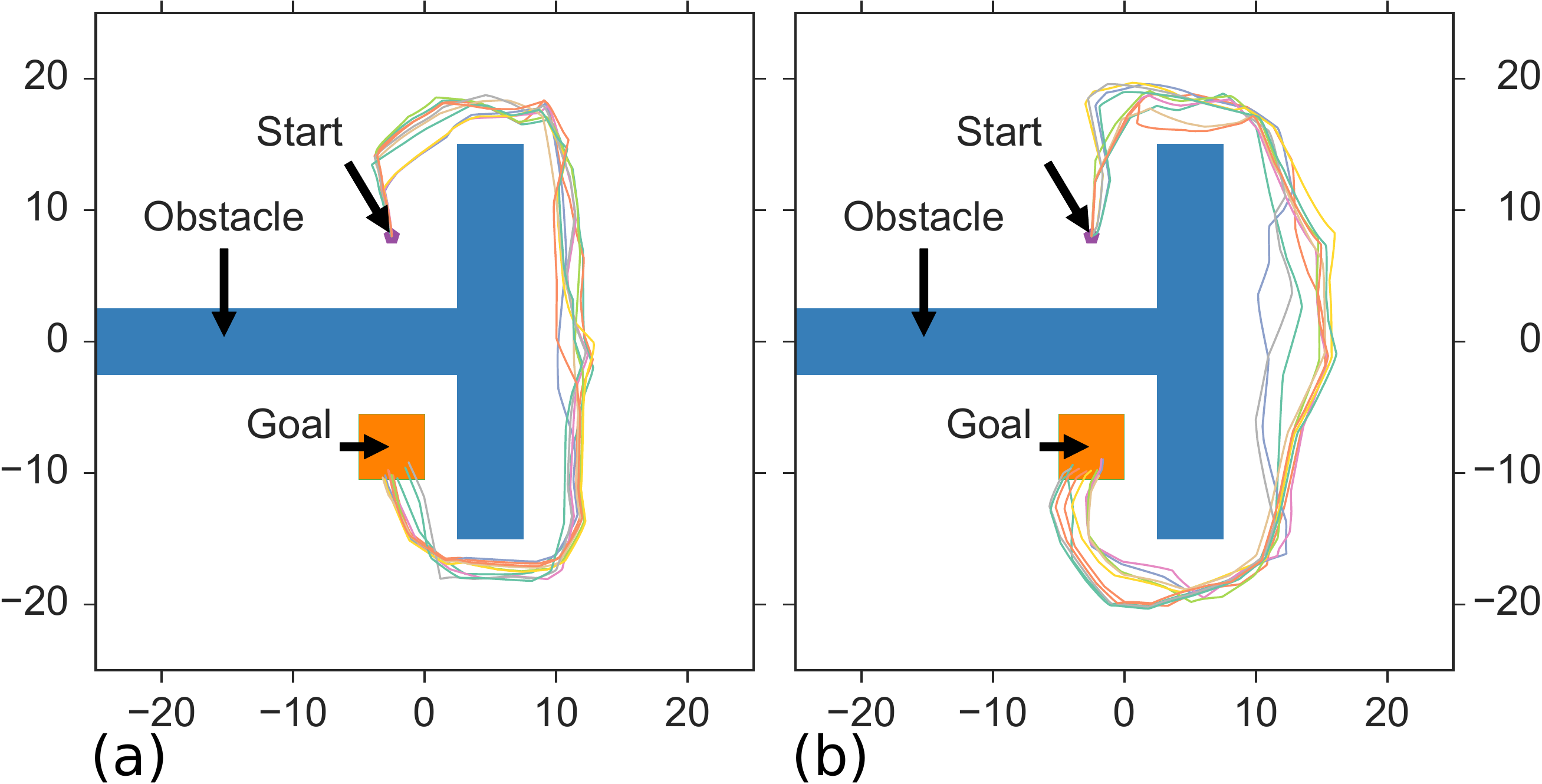}
        \caption{
          (a) 10 samples of trajectories generated via Rand with
          1000 states. (b) 10 samples of trajectories generated via BO
          with 1000 states. }  
        \label{fig:pushtraj}
\end{figure}

 We compare policies found by Rand and BO with the same set of sampled states ($|\tilde S|=200,400,600,800,1000$) within approximately the same
amount of total computation time.  They are both able to compute the
policy in $30\sim 120$ seconds, as shown in
Fig.~\ref{fig:pushtime}~(b). In more realistic domains, it is possible
that learning   the transition model will take longer and
dominate the action-selection computation. We simulate 100
trajectories in the Box2D simulator for each planned policy with a
maximum of 200 seconds. We show the result of the reward and success
rate in Fig.~\ref{fig:pushreward}, and the average
number of actions selected for visited states in
Fig.~\ref{fig:boaction}(b). In our simulations, BO consistently
performs approximately the same or better than Rand in terms of reward
and success rate while BO selects fewer actions than
Rand. We show 10 simulated trajectories for Rand and BO with
$|\tilde S|=1000$ in Fig.~\ref{fig:pushtraj}. 

From Fig.~\ref{fig:pushtime}~(a), it is not hard to see that RTDP successfully controlled the number of visited states to be only a small fraction of the whole sampled set of states. Interestingly, BO was able to visit slightly more states with RTDP and as a result, explored more possible states that it is likely to encounter during the execution of the policy, which may be a factor that contributed to its better performance in terms of reward and success rate in Fig.~\ref{fig:pushreward}. We did not compare with pure value iteration because the high computational cost of computing models for all the states made it infeasible.

\ALGNAME is able to compute models for only around 10\% of the sampled
states and about 200 actions per state. If we consider a naive grid
discretization for both action (3 dimension) and state (2 dimension)
with 100 cells for each dimension, the number of models we would have
to compute is on the order of $10^{10}$, compared to our approach,
which requires only $10^4$. 


\hide{

\zw{The following results will change}

In Fig.~\ref{fig:borand}, we compare the optimization process of $U=\max_a Q(s_0,a)$ via BO and random selection (Rand). We generated four different set of states via Algorithm~\ref{alg:statesample} with $52, 63,  146, 247$ number of states 
respectively. The same random actions were selected for both BO and Rand as known data before the evaluations. Comparing to random selection, BO performs much better at maximizing $Q(s_0,a)$, especially when more states are generated (Fig.~\ref{fig:borand}(d)). Usually if more states are generated, it creates more local optima for $Q(s,a)$, which makes it harder for Rand to find the global maximum value. On the contrary, Bayesian optimization adapts to the function it is optimizing, and approaches the global optima much faster than Rand.

We also compare Rand and BO on the task of finding trajectories for stochastic pushing system. The pusher is a point robot with rectangle shape, and the object is a cylinder. 

todo: add illustration of pusher and object here

An illustration of the environment is in Fig~\ref{fig:botraj}, where the blue patches are the obstacles, the orange circle is the goal region $\mathcal G$, the purple pentagon is the start position, and the pink dots are sampled states using Algorithm~\ref{alg:statesample}. The number of interior states sampled is $|\tilde S| = 141, 181, 221, 261, 301, 341, 381, 421$, and the number of boundary states sampled is $600,800,1000,1200,1400,1600,1800,2000$ correspondingly. The world is a $100\times 100$ square, the reward of hitting any obstacle or exiting the world is -100, and the reward for hitting the goal is 100. For the optimization in Line~\ref{opt1} and Line~\ref{opt2} of Algorithm~\ref{alg:rtdp}, we use BO/Rand to select 40 actions to compute the transition models and then choose the one with the highest value $U$. The transition models are cached and used for future selections. When more states are added, the cache will be deleted since the transition models did not take into account the new states, and hence inaccurate. However, we save a second cache for the GMM models computed before, and save computational cost of computing the GMM again. In Fig.~\ref{fig:botraj} and Fig.~\ref{fig:rdtraj}, we plot the trajectories and the contour of U found by RTDP with BO/Rand. BO finds trajectories that are more robust, while Rand is better at escaping from local optima if it is lucky enough. However, Rand does not have a resistant behavior unless a lot of samples of action are allowed. 

After getting the policy, we test the policy of BO and Rand with Monte Carlo simulation of the task. We plot the results of reward and success rate in Fig.~\ref{fig:bordpolicy} and ..
\begin{table}[t]
\centering
\begin{tabular}{|c|c|c|c|c|c|c|c|c|}
\hline
$|\tilde S|$   &    $141$ &    $181$ &    $221$ &    $261$  &   $301$  &  $341$ &   $381$ &   $421$ \\
\hline\hline
\multicolumn{9}{|c|}{Reward}\\\hline
Rand     & $-27. 3\rpm 15.2$ & $\textbf{-30.6}\rpm 13.3$ & $-30.3\rpm 13.8$ & $-20.8\rpm 8.0$  &  $\textbf{5.7}\rpm 6.5$  & $5.5\rpm 6.4$ &  $2.7\rpm 7.2$ &  $2.5\rpm 7.1$ \\
BO       &  $\textbf{-1.0}\rpm3.9$ &  $-0.3\rpm4.9$ &   $0.6\rpm 4.3$ &   $1.8\rpm4.8$  &  $1.2\rpm 5.5$  & $4.1\rpm 1.6$ &  $3.8\rpm 1.3$ &  $\textbf{4.8}\rpm 3.3$ \\

\hline\hline
\multicolumn{9}{|c|}{Success Rate}\\\hline

Rand     &   0.04 &   0.03 &   0.02 &   0.06  &  \textbf{\color{blue}0.92}  & 0.86 &  0.78 &  0.69 \\

BO       &   0.42 &   0.52 &   0.56 &   0.63  &  0.58  & \textbf{\color{red}1.00} &  \textbf{\color{red}1.00} &  0.97 \\
\hline

\end{tabular} 
\vspace{5pt}
\label{tb:bord}
\caption{BO performs better than Rand as expected for Monte Carlo simulation of reward collecting and its success rate of arriving at the goal without colliding with any obstacle. However, although Rand is less stable and prones to have results that oscillate more than BO, it can still achieve a relatively good performance without sophisticated modeling of the Q values of actions with GP. We also notice that the highest success rate does not necessarily mean the reward is also the highest, but a high reward usually means that the success rate will not be too low. We show the trajectories of the highest and lowest rewards achieved by BO and Rand in Fig.~\ref{fig:botraj} and Fig.~\ref{fig:rdtraj}}

\end{table}
\begin{figure}
        \centering
        \includegraphics[width=0.5\textwidth]{figs/u0bo11_10}
        \hspace{-10pt}
        \includegraphics[width=0.5\textwidth]{figs/u0bo31_10}

        \includegraphics[width=0.5\textwidth]{figs/u0bo21_10}
        \hspace{-10pt}
        \includegraphics[width=0.5\textwidth]{figs/u0bo51_10}
        \caption{BO performs better than Rand at optimizing Q function for continuous action space. }
        \label{fig:borand}
\end{figure}

\begin{figure}
        \centering
        \includegraphics[width=1\textwidth]{figs/boi_rd_0_exp_1_usefake_0_limitpre_0_policy_29}
        \caption{Trajectories and $U$ contour of BO with $|\tilde S|=141$ (lowest reward) and $|\tilde S| = 421$ (highest reward). }
        \label{fig:botraj}
\end{figure}

\begin{figure}
        \centering
        \includegraphics[width=1\textwidth]{figs/boi_rd_1_exp_1_usefake_0_limitpre_0_policy_36}
        \caption{Trajectories and $U$ contour of Rand with $|\tilde S|=181$ (lowest reward) and $|\tilde S| = 301$ (highest reward). }
        \label{fig:rdtraj}
\end{figure}

\begin{figure}
        \centering
        \includegraphics[width=0.5\textwidth]{figs/boi_rd_0_exp_1_usefake_0_limitpre_0_usedata_0_reward}
        \hspace{-10pt}
        \includegraphics[width=0.5\textwidth]{figs/boi_rd_0_exp_1_usefake_0_limitpre_0_usedata_0_succrate}
        \caption{The policy of BO behaves better and more stable than the policy of Rand. }
        \label{fig:bordpolicy}
\end{figure}

\begin{figure}
\centering
        \includegraphics[width=0.8\textwidth]{figs/boi_rd_0_exp_1_usefake_0_limitpre_0_usedata_0_cachesize}
\label{fig:bordcache}
\caption{BO does not need as many Q function calls as Rand to achieve a better policy.}
\end{figure}

\hide{
\subsection{Effectiveness of RRT state sampling}
Compare RRTRTBOMDP and Grid RTBOMDP (or Grid BOMDP??)
\subsection{Effectiveness of RTDP}
Compare the MC value for the policy from RRTRTBODP and RRTBODP.

\section{Previous results}
200 sampled states in the inner part of $S$. 45 iterations in the outer loop of value iteration. $\gamma = 0.9$. $R(\text{goal}|\cdot, \cdot) = 10$, $R(\text{obstacle}|\cdot, \cdot) = -50$, for other states reward is 0. Goal state is (40,40).
\begin{figure}
        \centering
        \includegraphics[width=1.3\textwidth]{figs/policy}
        \hspace{20pt}
        \includegraphics[width=1.3\textwidth]{figs/vi}
        \caption{Policy and value function for sampled states.}
        \label{fig:gpexample}
\end{figure}

\begin{figure}
        \centering
        \includegraphics[width=\textwidth]{figs/adensity_before}
        \caption{Density estimation for policy before value iteration.}
        \label{fig:gpexample}
\end{figure}

\begin{figure}
        \centering
        \includegraphics[width=\textwidth]{figs/adensity_after}
        \caption{Density estimation for policy after value iteration.}
        \label{fig:gpexample}
\end{figure}
}
}
\hide{\section{Discussion}
In this section we clarify the differences between our method and iMDP. Our method builds on iMDP and makes significant improvements to overcome the practical challenges including non-reversible dynamics, non-Gaussian transition models, and the non-convex optimization for bellman updates in the continuous space. Beyond the fact that iMDP assumes a known model for the transition dynamics and we do not, there are three major algorithmic differences between our method and iMDP: (1) our planner is highly modular, which separates the state sampling module and the dynamic programming module, so that more flexible state sampling techniques could be used without assuming the backward transition model is given; while iMDP combine state sampling (one more state each iteration that requires backward transition model to compute) with bellman updates among the neighbors of the newly sampled state. (2) iMDP computes the maximum over actions via uniformly random sampling, assuming that the Bellman equation can be
   accurately solved in the limit of infinitely many action samples;
   in our case, we are concerned about practical performance, and so
   perform a more utility-driven action-sampling strategy via a variant of 
   Bayesian optimization. (3) iMDP uses value iteration while our method use RTDP. We make this choice because the model can be expensive to compute, and we
   would like to compute as few models for state-action pairs as
   possible. Because of the modularity of our approach, other types of heuristics can also be used to replace RTDP, e.g. bounded RTDP~\cite{mcmahan2005bounded,smith2006focused}. Other differences between our method and iMDP include the type of RRT
for state sampling (backward extension for iMDP and forward extension for our method) and the domain of
interest for the stochastic models and planning problem (single-modal
transitions, multiple starting states and single goal for iMDP and
multi-modal transitions, multiple start-goal pairs for our method).}
\section{Conclusion}


An important class of robotics problems are intrinsically continuous in both state and action space, and may demonstrate non-Gaussian stochasticity in their dynamics. We have provided a framework to plan and learn effectively for these problems. We achieve efficiency by focusing on relevant subsets of state and action 
spaces, while retaining guarantees of asymptotic optimality.



\bibliographystyle{plain}
\bibliography{IEEEabrv,refs}


\end{document}